\documentclass{colt2013} % Anonymized submission
\title[Sharp analysis of low-rank 
kernel matrix approximations]{Sharp analysis of low-rank 
kernel matrix approximations}

 \coltauthor{\Name{Francis Bach} \Email{francis.bach@ens.fr}\\
 \addr INRIA - Sierra project-team\\
D\'epartement d'Informatique de l'Ecole Normale Sup\'erieure \\
Paris, France 
  }

\usepackage{amssymb,amsmath,color}
\usepackage{url,times}

\newcommand{\BEAS}{\begin{eqnarray*}}
\newcommand{\EEAS}{\end{eqnarray*}}
\newcommand{\BEA}{\begin{eqnarray}}
\newcommand{\EEA}{\end{eqnarray}}
\newcommand{\BEQ}{\begin{equation}}
\newcommand{\EEQ}{\end{equation}}
\newcommand{\BIT}{\begin{itemize}}
\newcommand{\EIT}{\end{itemize}}
\newcommand{\BNUM}{\begin{enumerate}}
\newcommand{\ENUM}{\end{enumerate}}
\newcommand{\BA}{\begin{array}}
\newcommand{\EA}{\end{array}}
\newcommand{\diag}{\mathop{\rm diag}}

\newcommand{\tr}{\mathop{ \rm tr}}

\newcommand{\idm}{I}
\newcommand{\rb}{\mathbb{R}}
\newcommand{\PP}{\mathbb{P}}

\newenvironment{proofsketch}{\par\noindent{\bf Proof sketch.\ }}{\hfill\BlackBox\\[2mm]}

\newcommand{\mysec}[1]{Section~\ref{sec:#1}}
\newcommand{\eq}[1]{Eq.~(\ref{eq:#1})}
\newcommand{\myfig}[1]{Figure~\ref{fig:#1}}

\def \E{{\mathbb E}}

\begin{document}

\maketitle

\begin{abstract}
We consider supervised learning problems within the positive-definite kernel framework, such as kernel ridge regression, kernel logistic regression or the support vector machine. With  kernels leading to infinite-dimensional feature spaces, a common practical limiting difficulty is the necessity of computing the kernel matrix, which most frequently leads to algorithms with running time at least quadratic in the number of observations $n$, i.e., $O(n^2)$. Low-rank approximations of the kernel matrix are often considered as they allow the reduction of  running time complexities  to $O(p^2 n)$, where~$p$ is the rank of the approximation. The practicality of such methods thus depends on the required rank $p$. In this paper, we show that in the context of kernel ridge regression, for approximations based on a random subset of columns of the original kernel matrix, the rank $p$ may be chosen to be linear in the \emph{degrees of freedom} associated with the problem, a quantity which is classically used in the statistical analysis of such methods, and is often seen as the implicit number of parameters of non-parametric estimators. This result enables simple algorithms that have sub-quadratic running time complexity, but provably exhibit the same \emph{predictive performance} than existing algorithms, for any given problem instance, and not only for worst-case situations.
\end{abstract}
 
%\begin{keywords}
%List of keywords
%\end{keywords}

\section{Introduction}

Kernel methods, such as the support vector machine or kernel ridge regression, are now widely used in many areas of science and engineering, such as computer vision or bioinformatics~\citep[see, e.g.,][]{scholkopf2004kernel,zhang2006local}. Their main attractive features are that (1) they allow non-linear predictions through the same algorithms than for linear predictions, owing to the kernel trick; (2) they allow the separation of the representation problem  (designing good kernels for non-vectorial data) and the algorithmic/theoretical problems (given a kernel, how to design, run efficiently and analyze estimation algorithms). Moreover, (3) their applicability goes beyond supervised learning problems, through the kernelization of classical unsupervised learning techniques such as principal component analysis or K-means. Finally, (4) probabilistic Bayesian interpretations through Gaussian processes allow their simple use within larger probabilistic models. For more details, see, e.g.,~\citet{GP,smola-book,Cristianini2004}.

However, kernel methods typically suffer from at least quadratic running-time complexity in the number of observations~$n$, as this is the complexity of computing the kernel matrix. In large-scale settings where $n$ may be  large, this is usually not acceptable. In these situations where plain kernel methods cannot be run, practitioners would commonly (a) turn to methods such as boosting, decision trees or random forests, which have both good running time complexity and predictive performance. However, these methods are typically run on data coming as vectors and usually put a strong emphasis on a sequence of decisions based on single variables. Another common solution is (b) to stop using infinite-dimensional kernels and restrict the kernels to be essentially linear kernels (i.e., by choosing an explicit representation of the data whose size is independent of the number of observations) where the non-parametric kernel machinery (of adapting the complexity of the underlying predictor to the size of the dataset) is lost, and the methods may then \emph{underfit}.

In this paper, we consider the traditional kernel set-up for supervised learning, where the input data are only known through (portions of) the kernel matrix. The main question we try to tackle is the following: Is it possible to run supervised learning methods with positive-definite kernels in time which is subquadratic in the number of observations without losing predictive performance? Of course, if adaptation is desired, linear complexity seems impossible, and therefore we should expect (hopefully slightly) super-linear algorithms. Statistically, a quantity that characterizes the non-parametric nature of kernel method is the \emph{degrees of freedom}, which play the role of an implicit number of parameters and which we  define and review in \mysec{df}. This quantity allows to go beyond worst-case analyses which are common in statistical learning theory: our generalization bounds will then depend on problem-dependent quantities which may not be known at training time, but that characterize finely the behavior on any given problem instance, and not only for the worst case over a large class of problems. In this paper, we try to tackle the following specific question: Do the degrees of freedom play a role in the computational properties of kernel methods?

An important feature of kernel matrices is that they are positive-semidefinite, and thus they may well be approximated from a random subset of $p$ of their columns, in running-time complexity $O(p^2n)$ and with a computable bound on the error (see details in \mysec{column}). This appears through different formulations within numerical linear algebra or machine learning, e.g., Nystr\"om method~\citep{williams2001using}, sparse greedy approximations~\citep{smola2000sparse}, incomplete Cholesky decomposition~\citep{fine01efficient,csi}, Gram-Schmidt orthonormalization~\citep{Cristianini2004} or CUR matrix decompositions \citep{mahoney2009cur}. It has been thoroughly analyzed in contexts where the goal is kernel matrix approximation or approximate eigenvalue decomposition~\cite[see, e.g.,][]{boutsidis2009improved,mahoney2009cur,ameet,gittens2011spectral}.
Such bounds have also been subsequently used to characterize the approximation of predictions made from these low-rank decompositions~\citep{cortes2010impact,jin}, but these two-stage analyses do not lead to guarantees that reflect the good observed practical behavior.
In this paper, our analysis aims at answering explicitly the simple question: how big should $p$ be to incur no loss of predictive performance compared to the full kernel matrix? The key insight of this paper is not to try to approximate the kernel matrix  well, but to predict well from the approximation. This requires a sharper analysis of the approximation properties of the column sampling approach.

\vspace*{.1cm}

We make the following contributions:
\begin{list}{\labelitemi}{\leftmargin=1.7em}
   \addtolength{\itemsep}{-.215\baselineskip}
\item[--]
In the fixed design least-squares regression setting, we show in \mysec{theorem} that the rank $p$ can be chosen to be linear in the \emph{degrees of freedom} associated with the problem, a quantity which is classically used in the statistical analysis of such methods. Note that our results hold for any problem instance, and not only in a worst-case regime.

\item[--] We present in \mysec{algorithms} simple algorithms that have sub-quadratic running time complexity, and, for the square loss, provably exhibit the same \emph{predictive performance} as classical algorithms than run in quadratic time (or more). 

\item[--] We provide  in \mysec{decay} explicit examples of optimal values of the regularization parameters and the resulting degrees of freedom, as functions of the decay of the eigenvalues of the kernel matrix, shedding some light in the joint computational/statistical trade-offs for choosing a good kernel. In particular, we show that with kernels with fast spectrum decays (such as the  Gaussian kernel), computational limitations may prevent exploring the relevant portions of the regularization paths, leading to underfitting.
\end{list}

\section{Supervised learning with positive-definite kernels}

In this section, we present the problem we try to solve, as well as several areas of the machine learning and statistics literatures our method relates to.

\subsection{Equivalent formulations}
\label{sec:eq}

Let $(x_i,y_i)$, $i=1,\dots,n$, be $n$ pairs of points in $\mathcal{X} \times \mathcal{Y}$, where $\mathcal{X}$ is the input space, and $\mathcal{Y}$ is the set of outputs/labels. In this paper,  we consider the problem of minimizing
\BEQ
\label{eq:primal}
\min_{f \in \mathcal{F}} \frac{1}{n} \sum_{i=1}^n \ell(y_i,f(x_i) ) + \frac{\lambda}{2} \| f\|^2,
\EEQ
where $\mathcal{F}$ is a reproducing kernel Hilbert space with feature map $\phi: \mathcal{X} \to \mathcal{F}$, and positive-definite kernel $k: \mathcal{X} \times \mathcal{X} \to \rb$.  While this problem is formulated as an optimization problem in a Hilbert space, it may be formulated as the optimization over $\rb^n$ in two different ways.

First, using the representer theorem, the unique solution $f$ may be found as $f = \sum_{i=1}^n \alpha_i \phi(x_i)$ \citep[see, e.g.,][]{wahba,smola-book,Cristianini2004}. Thus, by replacing the expression of $f$ in \eq{primal}, $\alpha$ is a solution of the following optimization problem:
\BEQ
\label{eq:primal-rep}
\min_{ \alpha \in \rb^n } \frac{1}{n} \sum_{i=1}^n \ell(y_i,(K \alpha)_i ) + \frac{\lambda}{2} \alpha^\top K \alpha,
\EEQ
 where $K \in \rb^{n \times n} $ is the \emph{kernel matrix}, defined as $K_{ij} = k(x_i,x_j)$.

Second, for convex losses only, an equivalent dual problem is classically obtained as (see proof in Appendix~\ref{app:dual}):
\BEQ
\label{eq:dual}
\max_{\alpha \in \rb^n} -g(-\lambda \alpha) - \frac{\lambda}{2} \alpha^\top K \alpha,
\EEQ
where $g(z) = \max_{ u \in \rb^n} 
 - \frac{1}{n} \sum_{i=1}^n \ell(y_i,u_i )  + u_i z_i
$ is the Fenchel-conjugate of the empirical risk (for the hinge loss, \eq{dual} is exactly the classical dual formulation of the SVM). Again, one may express the primal solution as 
$f = \sum_{i=1}^n \alpha_i \phi(x_i)$. In many situations (such as with the square loss or logistic loss), then the solution of \eq{dual} is unique, and it is also a solution of \eq{primal-rep} (note however that the converse is not true).

\subsection{Related work}

\paragraph{Efficient optimization algorithms for kernel methods.} In order to solve \eq{primal}, algorithms typically consider a primal or a dual approach.
Solving \eq{primal-rep}, i.e., the primal formulation after application of the representer theorem, is typically inefficient because the problem is ill-conditioned\footnote{The objective function in \eq{primal-rep} is a function of $K^{1/2} \alpha$, with a kernel matrix $K$ which is often ill-conditioned, usually leading to ill-conditioning of the original problem~\citep{chapelle2007training}.} and thus   second-order algorithms are typically used~\citep{chapelle2007training}. Alternatively, $K$ is represented explicitly as $K = \Phi \Phi^\top$ and a change of variable $ w = \Phi^\top \alpha$  is considered (note that when the kernel $k$ is linear, $\Phi$ is simply the design matrix, and we are solving directly a linear supervised learning problem). Then, the classical battery of convex optimization algorithms may be used, such as gradient descent, stochastic gradient descent~\citep{shalev2007pegasos} or cutting-planes~\citep{cutting}. However, in a kernel setting where a small matrix $\Phi$ (i.e., with few columns) is not known a priori, then they all exhibit at least quadratic complexity in $n$, as the full kernel matrix is used.

The dual problem in \eq{dual} is usually better-behaved (it has a better condition number)~\citep{chapelle2007training}, and algorithms such as coordinate descent and its variants such as sequential minimal optimization may be used~\citep{platt1999fast}. Again, in general, the full kernel matrix is needed.

Some algorithms operate online and do not need to compute the full kernel matrix, such as the ``forgetron''~\citep{dekel2005forgetron},  the ``projectron''~\citep{orabona2008projectron},   BGSD~\citep{wang2012breaking}, or LASVM~\citep{bordes2005fast}, with typically a fixed computational budget and good practical performance.
They often come with theoretical approximation guarantees, which are either data-dependent or based on worst-case analysis; however, these do not characterize the required rank which is needed to achieve the same accuracy than the problem with a full kernel matrix. In fact, one of the main motivations for this work is to derive precise bounds for reduced-set stochastic gradient algorithms for supervised kernel problems.

\paragraph{Analysis of column sampling approximation.}
Given a positive semi-definite matrix $K$ of size~$n$,   many methods exist for approximating it with a low-rank (typically also positive semidefinite) matrix $L$. While the optimal approximation is obtained from the eigenvalue decomposition, it is not computationally efficient as it has complexity at least quadratic in $n$ (since it requires the knowledge of $K$). In order to achieve linear complexity in $n$, approximations from subsets of columns are considered and appear under many names: Nystr\"om method~\citep{williams2001using}, sparse greedy approximations~\citep{smola2000sparse}, incomplete Cholesky decomposition~\citep{fine01efficient}, Gram-Schmidt orthonormalization or CUR matrix decompositions~\citep{mahoney2009cur}. Note that reduced-set methods~\citep[see, e.g.,][]{keerthi2006building} typically consider using a subset of columns after the predictor has been estimated. These low-rank methods are described in \mysec{column} and have running time complexity $O(p^2 n)$ for an approximation of rank $p$. Note that they may also be used in a Bayesian setting with Gaussian processes~\citep[see, e.g.,][]{lawrence2002fast}.

Column sampling has been analyzed a lot~\citep{mahoney2009cur,cortes2010impact,ameet,talwalkarmatrix,gittens2011spectral}; however, typically the analysis provides a high-probability bound on the error 
$\| K - L \|$ for an appropriate norm (typically operator, Frobenius or trace norm), but this is too pessimistic and does not really match with good practical performance (see empirical evidence in \myfig{pred}). Some works do consider prediction guarantees~\citep{cortes2010impact,jin}, but as shown in \mysec{theorem}, these are not sufficient to reach sharp results depending on the degrees of freedom. Moreover, many analyses consider situations where the  matrix $K$ is close to low-rank, which is not the case with kernel matrices. In this paper, the control of $K-L$ is more precise and adapted to the use of $K$ within a supervised learning method.

\paragraph{Randomized dimension reduction.}
The method presented in this paper, which considers random columns from the original kernel matrix, is also related to random projection techniques used for linear prediction problems~\citep{fot_mahoney,maillard2009compressed}. These techniques  are not kernel methods per se, as they require the knowledge of a matrix square root $\Phi$ (such that $K = \Phi \Phi^\top$), which leads to complexity greater than quadratic. For certain kernel functions that can be explicitly expressed as an expectation of dot-products in low-dimensional spaces, similar randomized dimensionality reduction may be performed~\citep{rahimi2007random}. Note however that the dimension reduction is then independent of the particular distributions of the input data points, while the column sampling approach is; see~\citet{jin_nys} for more discussion.

\paragraph{Theoretical analysis of predictive performance of kernel methods.}
In order to assess the required precision in approximating the kernel matrix, it is key to understand the typical predictive performance of kernel methods. For the square loss, this is classically obtained from a bias-variance decomposition of this performance (see
\mysec{analysis}). A key quantity is the \emph{degrees of freedom}, which play the role of an implicit number of parameters and is applicable to many non-parametric estimation methods which consists in ``smoothing'' the response vector by a linear operator~\citep[see, e.g.,][]{wahba,hastie_GAM,gu2002smoothing,Cap_DeV:2007,hsu2011analysis}. See precise definitions in \mysec{df}.

 \begin{figure}

 \vspace*{-.4cm}

\centering
 \includegraphics[scale=.6]{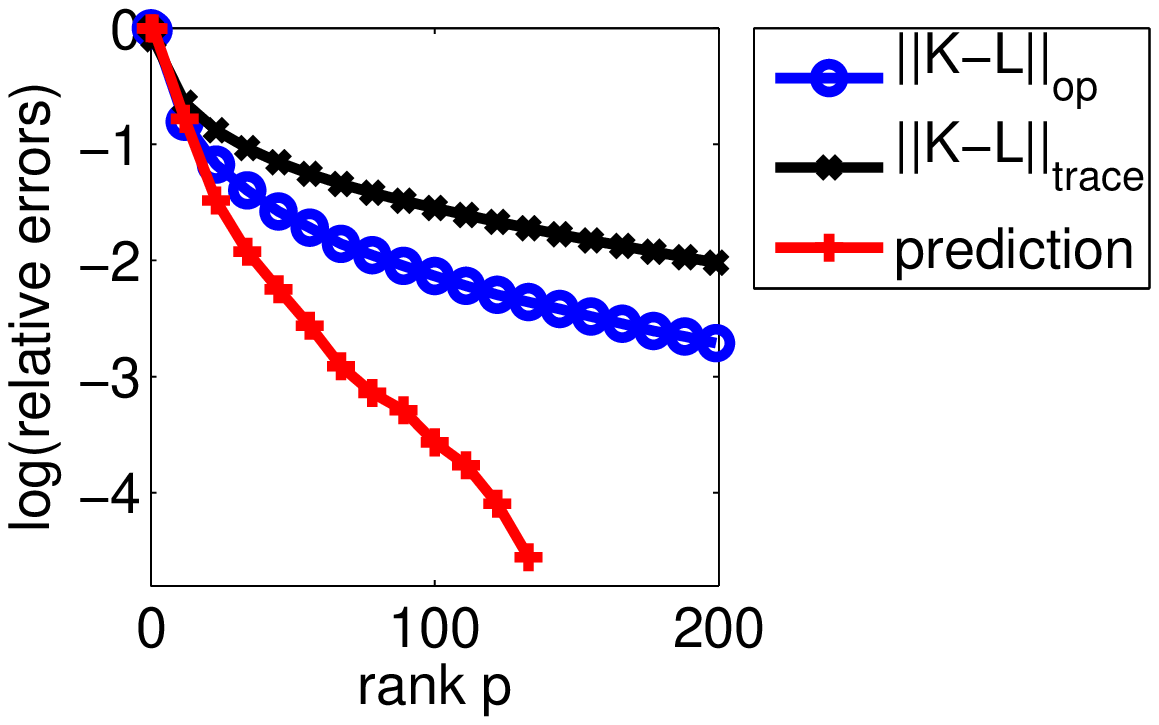} \hspace*{.5cm}
\includegraphics[scale=.6]{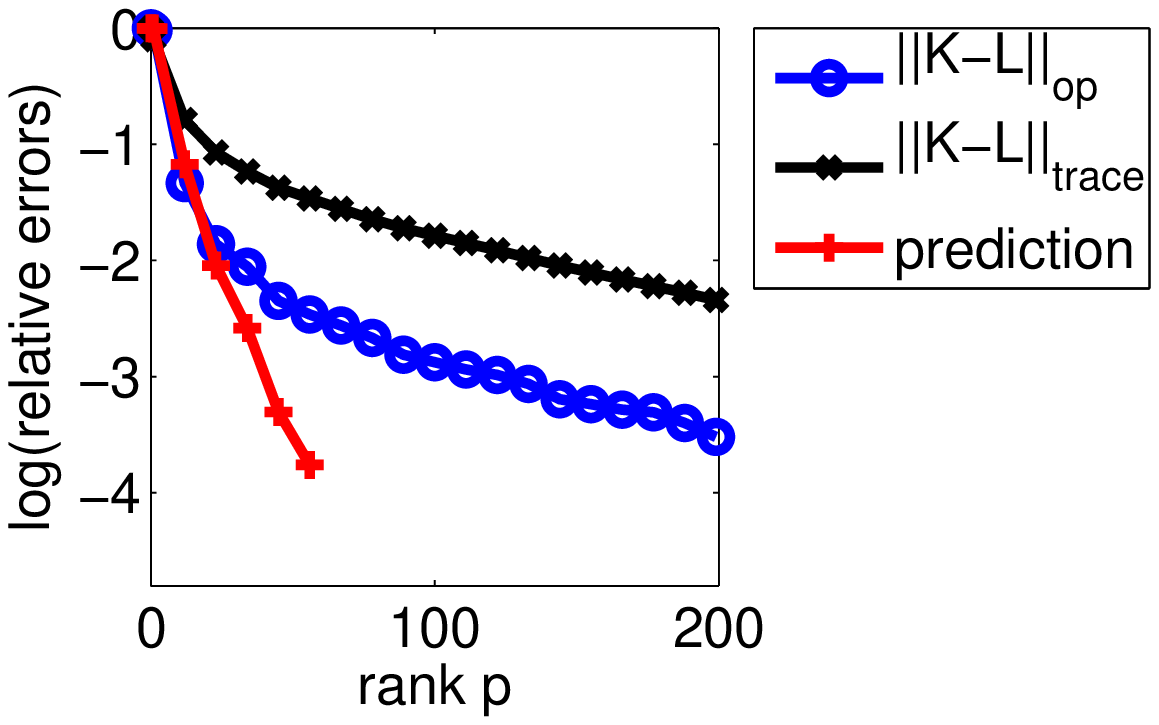} 
 
 \vspace*{-.4cm}

\caption{\textbf{Comparison of relative errors of kernel approximation and prediction performance.} We consider a least-squares prediction problem with $n=400$ and a decay of eigenvalues of the kernel matrix which is the inverse of a low-order polynomial (see examples in \mysec{simulations}). 
We compare the decays to zero of the relative kernel matrix approximation $\| K - L \|/ \|K\|$ (for the trace and operator norms) with the decay of the relative prediction performance (i.e., prediction for $L$ minus prediction for the full matrix $K$).
 Left: random selection of columns, right: selection of columns by incomplete Cholesky decomposition with column pivoting. The prediction error (red curve) stops when the average prediction error of the column sampling approach gets below the prediction error of the full kernel matrix approach.
 As the rank $p$ increases, the decay of the relative prediction error is much faster than the error in matrix approximation, suggesting that relying on good kernel matrix approximation may be suboptimal if the goal is simply to predict well.
}

 \vspace*{-.2cm}

\label{fig:pred}
\end{figure}

\section{Approximation from subset of columns}
\label{sec:column}

In this section, given subsets $A$ and $B$ of $V= \{1,\dots,n\}$ and a matrix $K \in \rb^{n \times n}$, we denote by $K(A,B)$ the submatrix of $K$ with rows in $A$ and columns in $B$.

\paragraph{Approximation from columns.}
Given a random subset $I$ of $V=\{1,\dots,n\}$ of cardinality $p$,  we simply consider the approximation of the kernel matrix $K$ from the knowledge of $K(V,I)$ (the columns of $K$ indexed by $I$), by the matrix 
\BEQ
\label{eq:L}
L = K(V,I) K(I,I)^\dagger K(I,V),
\EEQ
 where $M ^\dagger$ denotes the pseudo-inverse of $M$. As shown by~\citet{csi}, $L$ is the only symmetric matrix with column space spanned by the columns of $K(V,I)$, and such that $L(V,I) = K(V,I)$. Alternatively, given that $K$ is the matrix of dot-products of points in a Hilbert space, it may be seen as the kernel matrix of the  orthogonal projections of all points onto the affine subspace spanned by the points indexed by $I$~\citep{fot_mahoney}.

Note that approximating $K$ by $L$ also corresponds to creating an explicit feature map of dimension $p$, i.e., $\widetilde{\phi}(x) = K(I,I)^{-1/2}  ( k(x_i,x) )_{i \in I} \in \rb^p$, 
and, this allows the application to test data points (note that using such techniques also allows better \emph{testing} running time peformance).

Such a feature map may be efficiently obtained in running time $O(p^2n)$ using incomplete Cholesky decomposition---often interpreted as partial Gram-Schmidt orthonormalization, with the possibility of having an explicit online bound on the trace norm of the approximation error~\citep[see, e.g.,][]{Cristianini2004}.

\paragraph{Pivoting vs.~random sampling.} While selecting a random subset is computationally efficient, it may not lead to the best performance. For the task of approximating the kernel matrix, algorithms such as the incomplete Cholesky decomposition \emph{with pivoting}, provide an approximate greedy algorithm with the same complexity than random subsampling~\citep{smola2000sparse,fine01efficient}.

In \mysec{simulations}, we provide comparisons between the two approaches, showing the potential advantage of the greedy method over random subsampling. However, the analysis of such algorithms  is harder, and, to the best of our knowledge, still remains an open problem.

\section{Fixed design analysis for least-square regression (ridge regression)}

\label{sec:analysis}
To simplify the analysis,  we assume that the $n$  data points  $x_1,\dots,x_n$ are deterministic and that $\mathcal{Y} = \rb$.  In this setting, the classical \emph{generalization error} (prediction error on unseen data points) is replaced by the \emph{in-sample prediction error} (prediction error on observed data points). This fixed design assumption could be relaxed by using tools from~\citet{hsu2011analysis}, as results for random design settings are typically similar to the fixed design settings.

We assume that the loss $\ell$ is the square loss, i.e., $\ell(y_i,f(x_i)) = \frac{1}{2} ( y_i - f(x_i) )^2$. By using the representer theorem (see \mysec{eq}), we classically obtain:
$$
\textstyle f(x) = \sum_{i=1}^n \alpha_i k(x,x_i) \mbox{ with } \alpha = (K + n \lambda \idm)^{-1} y.
$$
This leads to a prediction vector $\hat{z} = K (K + n \lambda \idm)^{-1} y \in \rb^n$, which is a linear function of the output observations $y$, and is often referred to as a smoothed estimate of $z$.

\subsection{Analysis of the in-sample prediction error}

\label{sec:df}
We denote by $z_i = \E y_i \in \rb$ the expectation of $y_i$, and we denote by $\varepsilon_i = y_i - z_i = y_i - \E y_i \in \rb $  the noise variables; they have zero mean and are \emph{only} assumed to have finite covariance matrix $C$ (note that the noise may neither be independent nor identically distributed).

\paragraph{Bias/variance decomposition of the generalization error.}
Following classical results from the statistics literature~\citep[see, e.g.,][]{wahba,hastie_GAM,gu2002smoothing}, we obtain the following expected prediction error:
\BEAS
\textstyle \frac{1}{n} \E_\varepsilon   \| \hat{z} - z\|^2 & \!\!= \!\!& \textstyle
  \textstyle \frac{1}{n} \| \E_\varepsilon \hat{z} - z\|^2 +  \frac{1}{n} \tr {\rm var}_\varepsilon(\hat{z})  \\
  & \!\!= \!\!&  \textstyle \frac{1}{n} \| (\idm - K (K + n \lambda \idm)^{-1} ) z\|^2 +  \frac{1}{n}    \tr C
  K^2 (K + n \lambda \idm)^{-2} \\
   & \!\!= \!\! &\textstyle n  \lambda^2 z^\top (K + n \lambda \idm)^{-2} z  +  \frac{1}{n}    \tr C
  K^2 (K + n \lambda \idm)^{-2} ,
\EEAS
which may be classically decomposed in two terms:
\BEAS
{\rm bias}(K) & = &  n  \lambda^2 z^\top (K + n \lambda \idm)^{-2} z \\
{\rm variance}(K) & = &  \textstyle \frac{1}{n}  \tr C
  K^2 (K + n \lambda \idm)^{-2} .
\EEAS
Note that the bias term is a matrix-decreasing function of  $K/\lambda$ (and thus an increasing function of $\lambda$), while the variance term is a matrix-increasing function of $K/\lambda$ and the noise covariance  matrix~$C$.

\paragraph{Degrees of freedom.}
Note that an assumption which is usually made is $C = \sigma^2 \idm$;  the variance term then takes the form $\frac{\sigma^2}{n} \tr  
  K^2 (K + n \lambda \idm)^{-2}$ and $\tr  
  K^2 (K + n \lambda \idm)^{-2}$ is referred to as the \emph{degrees of freedom}~\citep{wahba,hastie_GAM,gu2002smoothing,hsu2011analysis} (note that an alternative definition is often used, i.e., $\tr  
  K (K + n \lambda \idm)^{-1}$, and that as shown in Appendix~\ref{app:lambda}, they often behave similarly). In   ordinary  least-squares estimation from $d$ variables, the variance term is equal to $\sigma^2 d/n$, and thus the degrees of freedom play the role of an implicit number of parameters. In this paper, we show that a proxy to this statistical quantity also plays a role in optimization: the number of columns needed to approximate the kernel matrix  precisely enough to incur no loss of performance is linear in the degrees of freedom.
  
  More precisely, we define the \emph{maximal marginal degrees of freedom} $d$ as
  \BEQ
  \label{eq:d}
  d = n   \big\|
\diag \big( K ( K + n \lambda \idm)^{-1}   \big) \big\|_\infty.
\EEQ
  We have $\tr  K^2 ( K + n \lambda \idm)^{-2}
  \leqslant \tr  K ( K + n \lambda \idm)^{-1} =  \big\|   
\diag \big( K ( K + n \lambda \idm)^{-1}   \big) \big\|_1 \leqslant d$, and thus~$d$ provides an upper-bound on the regular degrees of freedom. It may be significantly larger in situations where there may be outliers and the vector $\diag \big( K ( K + n \lambda \idm)^{-1}   \big) $ is far from uniform---precise results are out of the scope of this paper.
  Moreover, the diagonal elements of $   K ( K + n \lambda \idm)^{-1} $ are related to statistical leverage scores introduced for best-rank approximations~\citep{mahoney2009cur}; it would be interesting to see if this link could lead to non-uniform sampling schemes with better behavior.
  
  In \mysec{decay}, we study in detail how the degrees of freedom vary as a function of  $\lambda$ and $n$: in order to minimize predictive performance, the best choice of $\lambda$ depends on $n$ (as a decreasing function), typically smaller than a constant times $1/\sqrt{n}$, and the degrees of freedom typically grow as a slow function of $n$, reflecting the non-parametric nature of kernel methods.

 \subsection{Predictive performance of column sampling}
 \label{sec:theorem}

 We consider sampling $p$ columns  (without replacement) from the original $n$ columns. We consider the column sampling approximation defined in \eq{L} and provide sufficient conditions (a lower-bound on  $p$) to obtain the same predictive performance than  with the full kernel matrix.
 \begin{theorem}[Generalization performance of column sampling]
 \label{theo:perf}
 Assume $z \in \rb^n$ and $K \in \rb^{n \times n}$ are respectively a deterministic vector and a symmetric positive semi-definite matrix, and $\lambda>0$. Let $d =
 n   \big\|
\diag \big( K ( K + n \lambda \idm)^{-1}   \big) \big\|_\infty$ and $R^2 = \| \diag(K) \|_\infty$. Assume $\varepsilon \in \rb^n$ is a random vector with finite variance and zero mean, and define the smoothed estimate
 $\hat{z}_K = (K+ n\lambda \idm)^{-1}K(z+\varepsilon)$. Assume that $I$ is a uniform random subset of $p$ indices in $\{1,\dots,n\}$ and consider $L = K(V,I) K(I,I)^\dagger  K(I,V)$, with the approximate smoothed estimate $\hat{z}_L = (L+ n\lambda \idm)^{-1}L(z+\varepsilon)$. Let $\delta \in (0,1)$. If 
 \BEQ
 \label{eq:rank}
 p \geqslant \big( \frac{32 d }{\delta} + 2 \big) \log \frac{n R^2}{\delta \lambda},
 \EEQ
 then
 \BEQ
 \label{eq:bound}
 \frac{1}{n} \E_I \E_\varepsilon  \| \hat{z}_L - z\|^2
 \leqslant   ( 1 + 4 \delta)  \frac{1}{n}\E_\varepsilon  \| \hat{z}_K - z\|^2.
 \EEQ
 \end{theorem}

\begin{proofsketch}
The proof relies on approximating the expected error \emph{directly}, and not through bounding the error $\|K \!-\!L \|$. This is done by (a) considering a regularized version of $L$, i.e., $L_\gamma = K(V,I)  \big[ K(I,I) + p \gamma \idm\big]^{-1} K(I,V)$, (b) using a Bernstein inequality for an appropriately rescaled covariance matrix and (c) using monotonicity arguments to obtain the required bound. See more details in Appendix~\ref{app:proof}.
\end{proofsketch}

 \paragraph{Sharp relative approximation.}
The bound in \eq{bound} provides a relative approximation guarantee: the predictions $\hat{z}_L$ are shown to perform as well as $\hat{z}_K$ (no kernel matrix approximation). Small values of $\delta$ impose no loss of performance, while $\delta=1/4$ impose that the prediction errors have a similar behavior (up to a factor of $2$). Note that relative bounds may be more easily obtained by eigenvalue thresholding of the kernel matrix, i.e., through replacing soft-shrinkage by hard-thresholding~\citep{blanchard2004kernel,dhillon2011risk}. However, these bounds do not allow the proportionality constant to go arbitrarily close to one, and they depend on the full knowledge of the kernel matrix.

\paragraph{Lower bounds.}
The lower bound for the rank $p$ in \eq{rank} shows that the maximal marginal degrees of freedom provides a quantity which, up to logarithmic terms, is sufficient to scale with, in order to incur no loss of  prediction performance. Note that the previous result also allows the derivation of an approximation guarantee $\delta$ given a rank $p$, by inverting \eq{rank}. Moreover, Theorem~\ref{theo:perf} provides a sufficient lower-bound for the required rank $p$. Deriving precise necessary lower-bounds is outside the scope of this paper. However, given that with a reduced space of $p$ dimensions, we can achieve a prediction error of $O(p/n)$ from ordinary least-squares, we should expect $p$ to be larger than the known minimax rates of estimation for the problem at hand~\citep{johnstone1994minimax,Cap_DeV:2007,steinwart2009optimal}. In \mysec{decay}, we show that in some situations, it turns out that $d$ is of the order of the minimax rate; therefore, we could expect that in certain settings, $d$ is also a necessary lower-bound on $p$ (up to constants and logarithms).

\paragraph{High-probability results.}
The bound in \eq{bound} provides a result in expectation, both with respect to the data (i.e., $\E_\varepsilon$) and the sampling of columns (i.e., $\E_I$). While results in high probability with respect to $I$ are readily obtained since the proof is based on such results (see \eq{high} in Appendix~\ref{app:proof} for details), doing so with respect to $\varepsilon$  would require additional assumptions, which are standard in the analysis in ridge regression~\citep{hsu2011analysis,arlot}, but that would make the results significantly more complicated. 
 
 \paragraph{Avoiding terms in $1/\lambda$.}
Theorem~\ref{theo:perf} focuses on average predictive performance; this is different 
from achieving a good approximation of the kernel matrix~\citep{mahoney2009cur}. 
 Previous work~\citep{cortes2010impact,jin} considers explicitly the use of kernel matrix approximation bounds within classifiers or regressors, but obtains  bounds that involve  multiplicative terms
 of the form $1/\lambda$ or $1/\lambda^2$, which, as we show in \mysec{decay}, would grow as $n$ grows. More precisely, 
the bound from~\citet[Eq.~(5)]{cortes2010impact}
has a term of  the form $1/\lambda_{\min}(K+\lambda \idm)^2$; however, for the non-parametric problems we are considering in this paper, the lowest eigenvalue of $K$ is often below machine precision and hence the bound behaves as $1/\lambda^2$. Moreover, the
subsequent bound of~\citet[][Theorem 2]{cortes2010impact} has a term of the form $1/(\lambda^2 p)$ which can only be small if $p$ is   larger than $1/\lambda^2$, which, according to our analysis in \mysec{decay}, is typically larger than $n$ since $\lambda$ should typically decrease at least as $1/\sqrt{n}$ (however, note that their bound has a stronger nature than the one in Theorem~\ref{theo:perf}, as it states a guarantee in high probability).

Our proof technique, that focuses \emph{directly} on prediction performance
and side-steps the explicit approximation of the kernel matrix, avoids these terms in $1/\lambda$, and, beyond the dependence on $\lambda$ through the degrees of freedom (which we cannot avoid),  our dependence is only logarithmic in $\lambda$ (see details in the proof in Appendix~\ref{app:proof}). 

\paragraph{Instance-based guarantees.}
Theorem~\ref{theo:perf} shows that in the specific instance that we are faced with, we do not lose any average predictive performance. 
As opposed to~\citet{jin}, the bound is not on the worst-case predictive performance (obtained from optimizing over $\lambda$, and with worst-case analysis over $K$), but for  given $\lambda$ and $K$ (however, the bound of~\citet{jin} is a high-probability result while ours is only in expectation).
Moreover, even in this worst-case regime,~\citet{jin} state that for polynomial decays of the eigenvalues of $K$ such that the optimal prediction performance is of the form $O( n^{-1} n^{1/(\gamma+1)})$ (and for binary classification rather than regression), the rank to achieve this optimal prediction is $p=n^{2\gamma/(\gamma^2-1)}$, which  may be larger than $n$ and is significantly higher than
$ n^{1/(\gamma+1)} $ (which corresponds to our result since the degrees of freedom are then equal to $ n^{1/(\gamma+1)} $).  %\item[--] 

%\item[--] 
\paragraph{Link with eigenvalues.}
In the existing analysis of sampling techniques for kernel methods, another source of inefficiency which makes our result sharper is the   proof technique for bounding $\| K - L \|$. Indeed, most analyses use  a linear algebra lemma from~\citet{mahoney2009cur,boutsidis2009improved}, that relies on the $(p+1)$-th eigenvalue to be small; hence it is adapted to matrices with sharp eigenvalue decrease, which is not the case for kernel matrices (see an illustrative example in \myfig{pred}). We provide a new proof technique based on regularizing the column sampling approximation and optimizing the extra regularization parameter using a monotonicity argument.

%\item[--] 
\paragraph{Additional regularization effect.}
In our experiments, we have noticed that the low-rank approximation may have an additional regularizing effect leading to a \emph{better} prediction performance than with the full kernel matrix.

\paragraph{Beyond square loss.}  The notion of degrees of freedom can be extended to smooth losses like the logistic loss. However, the simple bias/variance decomposition only holds asymptotically, forcing a control of the two terms, which would lead to significant added complexity. 
%\end{list}

 \paragraph{Beyond fixed design.} In order to extend our analysis to random design settings, we would need to additionally control the deviation between  covariance operators and  empirical covariance operators, with quantities like the degrees of freedom that depend on the decay of non-zero eigenvalues and not on their number. This could be done using tools from~\cite{hsu2011analysis,hsu2}.

\subsection{Optimal choice of the regularization parameter}
\label{sec:decay}

As seen in Sections~\ref{sec:df} and \ref{sec:theorem}, the computational and statistical properties of kernel ridge regression depend heavily on the choice of the regularization parameter $\lambda$ as $n$ increases, which we now tackle.

For simplicity, in this section, we assume that the noise variables $\varepsilon$ are i.i.d. (i.e., $C = \sigma^2 \idm$). Our goal is to study simplified situations, where we can derive explicit formulas for the bias, the variance, and the optimal regularization parameter. Throughout this section, we will consider specific decays of certain sequences, which we characterize with the notation $u_n = \Theta(v_n)$, which means that there exist strictly positive constants $A$ and $B$ such that $A u_n \leqslant v_n \leqslant B v_n$ for all $n$.

We assume that the kernel matrix $K$ has eigenvalues of the form $\Theta(n \mu_i)$, $i=1,\dots,n$, for some summable sequence $(\mu_i)$---so that   $\tr K = \Theta (n)$, and that the coordinates of $z$ on the eigenbasis of~$K$ have the asymptotic behavior $ \Theta( \sqrt{n \nu_i})$ for a summable sequence $(\nu_i)$---so that $\frac{1}{n} z^\top z = \Theta(1)$. In Table~\ref{tab:decays}, we provide asymptotic equivalents of all quantities for several pairs of sequences $(\mu_i)$ and $(\nu_i)$ (see proofs in Appendix~\ref{app:lambda}), with polynomial or exponential decays. 

Note that for decays of $\nu_i$ which are polynomial, i.e., $\nu_i = O(i^{-2\delta})$, then the best possible prediction performance is known to be $O(n^{1/2\delta-1})$~\citep{johnstone1994minimax,Cap_DeV:2007} and is achieved if the RKHS is large enough (lines 2 and 4 in Table~\ref{tab:decays}). For exponential decay, the best performance   is $O( \log n/n)$. See also~\citet{steinwart2009optimal}.

Given a specific decay $(\nu_i)$ for expected outputs $z = \E y$, then depending on the decay $(\mu_i)$ of the eigenvalues of the kernel matrix, the final prediction performance and the optimal regularization parameter may be different. Usually, the smaller the  RKHS, the faster the decay of eigenvalues of the kernel matrix $K$---this is true for translation-invariant kernels~\citep{smola-book}, and the kernels considered in \mysec{sobolev}. Thus there are two regimes:

\begin{list}{\labelitemi}{\leftmargin=1.7em}
   \addtolength{\itemsep}{-.215\baselineskip}
\item[--] \textbf{The RKHS is too large}, for lines 1 and 3 in Table~\ref{tab:decays}: the eigenvalues of $K$, which depend linearly on $\mu_i$, do not decay fast enough. In other words, the functions in the RKHS are not smooth enough. In this situation, the prediction performance is suboptimal (do not attain the best possible rate).

\item[--] \textbf{The RKHS is too small}, for lines 2, 4, and 6 in Table~\ref{tab:decays}: the eigenvalues of $K$ decay fast enough to get an optimal prediction performance. In other words, the functions in the RKHS are potentially smoother than what is necessary. In this situation however, the required value of~$\lambda$ may be very small (much smaller than $O(n^{-1})$), leading to potentially harder optimization problems (since the condition number that  depends on $1/\lambda$ may be very large).
\end{list}

There is thus a computational/statistical trade-off: if the RKHS is chosen too large, then the prediction performance is suboptimal (i.e., even with the best possible regularization parameter, the resulting error is not optimal); if the RKHS is chosen too small, the prediction performance could be optimal, but the optimization problems are harder, and sometimes cannot be solved with the classical precision of numerical techniques (see examples of such behavior in \mysec{simulations}). Indeed,  for least-squares regression, where a positive semi-definite linear system has to be solved, its condition number is proportional to $1/\lambda$ and for the best choice of $\lambda$ in line 4 of 
Table~\ref{tab:decays}, it grows exponentially fast in $n$ (if $\lambda$ is chosen larger, then the bias term will lead to sub-optimal behavior).

\begin{table}

 \vspace*{-.4cm}

\centering
\hspace*{-.5cm}
\begin{tabular}{|l|l|l|l|l|l|l|l|}
\hline
$(\mu_i)$& $(\nu_i)$ & var.& bias & optimal $\lambda$ & pred. perf. &  d.f. $d_{ave}$ & condition    \\[.05cm]
\hline
\\[-.7cm]
& & & & &  &  & \\
$  i^{-2\beta} $  & $ i^{-2\delta} $  & $  n^{-1} \lambda^{-1/2\beta}$ & $  \lambda^2$  &  
$ n^{-1/(2+ 1/2\beta)}$ & $n^{1/(4 \beta+ 1)-1}$ & $n^{1/(4 \beta+ 1)}$  &
if $2\delta \!> \!4 \beta \!+\! 1$\\[.1cm]
$  i^{-2\beta} $  & $ i^{-2\delta} $  & $  n^{-1}\lambda^{-1/2\beta}$ & $ \lambda^{(2 \delta -1)/2 \beta}$  & $n^{-\beta/\delta}$ &  $n^{1/(2 \delta) - 1}$ &  $n^{1/(2 \delta) }$  &
if $2\delta \! <\! 4 \beta\! +\! 1$\\[.1cm]
$   i^{-2\beta} $  & $e^{-\kappa i}$  & $ n^{-1}\lambda^{-1/2\beta}$ & $  \lambda^2 $ & $ n^{-1/(2+ 1/2\beta)}$ & $n^{1/(4 \beta+ 1)-1}$ & $n^{1/(4 \beta+ 1)}$ &  \\[.1cm]
$ e^{-\rho i} $  & $  i^{-2\delta} $  &   $n^{-1}\log \frac{1}{\lambda}$  & $  ( \log \frac{1}{\lambda})^{1-2\delta}$&$ \exp( - n^{1/(2\delta)} )$  &$n^{1/(2 \delta) - 1}$  &$n^{1/(2 \delta) }$   & \\[.1cm]
$  e^{-\rho i} $  & $e^{-\kappa i}$  & $n^{-1}\log \frac{1}{\lambda}$ &  $ \lambda^2$  &  $n^{-1/2}$& $\log n/n$   & $\log n$   & if $\kappa > 2\rho$ \\[.1cm]
$  e^{-\rho i} $  & $e^{-\kappa i}$  & $n^{-1}\log \frac{1}{\lambda}$ & $\lambda^{\kappa/\rho} $ & $n^{-\rho/\kappa}$ & $\log n /n $ & $\log n  $  & if $\kappa <2 \rho$ \\[.1cm]
\hline
\end{tabular}

 \vspace*{-.1cm}

\caption{
Variance, bias, optimal regularization parameter, corresponding prediction performance and degrees of freedom $d_{ave} = \tr K^2 ( K + n\lambda \idm)^{-2}$, for several decays of eigenvalues and signal coefficients (we always assume $\delta>1/2$, $\beta>1/2$, $\rho>0$, $\kappa>0$, to make the series summable). All entries are functions of $i$, $n$ or $\lambda$ and are only asymptotically bounded below and above, i.e., correspond to the asymptotic notation $\Theta(\cdot)$.  }
\label{tab:decays}

 \vspace*{-.2cm}

\end{table}

\subsection{Optimization algorithms with column sampling}
\label{sec:algorithms}
 
Given a rank $p$ and a regularization parameter $\lambda$, we consider the following algorithm to solve \eq{primal} for twice differentiable convex losses: 
\begin{list}{\labelitemi}{\leftmargin=1.7em}

  \addtolength{\itemsep}{-.215\baselineskip}
\item[1.] Select at random $p$ columns of $K$ (without replacement).
\item[2.] Compute $\Phi \in \rb^{n \times p}$ such that $\Phi \Phi^\top = K(V,I) K(I,I)^\dagger K(I,V)$ using incomplete Cholesky decomposition~\citep[see details in][]{Cristianini2004}.
\item[3.] Minimize $\min_{w \in \rb^p } \frac{1}{n} \sum_{i=1}^n \ell(y_i, (\Phi w)_i) + \frac{\lambda}{2} \| w\|^2$ using Newton's method (i.e., a single linear system for the square loss).
\end{list}
The complexity of step 2 is already $O(p^2 n)$, therefore using faster techniques for step 3 (e.g., accelerated gradient descent) does not change the overall complexity, which is thus $O(p^2n)$. Moreover, since we use a second-order method for step 3, we are robust to ill-conditioning and in particular to small values of $\lambda$ (though not below machine precision as seen in \mysec{simulations}). This is not the case for algorithms that relies on the strong convexity of the objective function, whose convergence is much slower when $\lambda$ is small (as seen in \mysec{decay}, when $n$ grows, the optimal value of $\lambda$ can decay very rapidly, making these traditional methods non robust).

According to Theorem~\ref{theo:perf}, at least for the square loss, the dimension $p$ may be chosen to be linear in the maximal marginal degrees of freedom $d$ defined in \eq{d}; it  is in practice close to the traditional degrees of freedom $d_{ave}$, which, as illustrated in Table~\ref{tab:decays}, is typically smaller than $n^{1/2}$. Therefore, if $p$ is properly chosen, the complexity is subquadratic. Given $\lambda$, $d$ (and thus $p$) can be estimated from a low-rank approximation of $K$. However, our current analysis assumes that $\lambda$ is given. Selecting the rank $p$ \emph{and} the regularization parameter $\lambda$ in a data-driven way would make the prediction method more robust, but this would require extra assumptions~\citep[see, e.g.,][and references therein]{arlot}.

\section{Simulations}
\label{sec:simulations}
\label{sec:experiments}

\label{sec:sobolev}

\paragraph{Synthetic examples.}
In order to study various behaviors of the regularization parameters $\lambda$ and the degrees of freedom $d$, we consider periodic smoothing splines on $[0,1]$ and points $x_1,\dots,x_n$   uniformly spread over $[0,1]$, either deterministically or randomly. In order to generate problems with given sequences $(\mu_i)$ and $(\nu_i)$, it suffices to choose
$k(x,y) = \sum_{i=1}^\infty 2 \mu_i \cos 2 i\pi (x-y)$, and a function $f(x)
= \sum_{i=1}^\infty 2 \nu_i^{1/2} \cos 2 i\pi x $. For $\mu_i = i^{-2\beta}$, we have $k(x,y) = \frac{1}{(2\beta)!}
B_{2 \beta}(x-y - \lfloor x - y\rfloor) $,
 where
 $B_{2\beta}$ is the $(2\beta)$-th Bernoulli polynomial (see details in Appendix~\ref{app:kernels}).

\paragraph{Optimal values of $\lambda$.} In a first experiment, we illustrate the results from \mysec{decay}, and compute in \myfig{rkhs} the best value of the regularization parameter (left) and the obtained predictive performance (middle), for a problem with $\nu_i = i^{-2 \delta}$ for $\delta=8$, and for which we considered several kernels, for which $\mu_i = i^{-2 \beta}$, for $\beta = 1$,   $\beta = 4$ and   $\beta = 8$. \begin{list}{\labelitemi}{\leftmargin=1.7em}
   \addtolength{\itemsep}{-.215\baselineskip}
\item[--] For $\beta=1$, the  rate of convergence of $n^{1/(4 \beta+ 1)-1}$ happens to be    achieved (line 1 in Table~\ref{tab:decays}), with a certain asymptotic decay of the regularization parameter, and it is slower than  $n^{1/(2\delta)-1}$.
\item[--] For $\beta=4$, the optimal rate  of $n^{1/(2\delta)-1}$ is achieved (line 2 in Table~\ref{tab:decays}), as expected.
\item[--] For $\beta = 8$, the rate of convergence should be $n^{1/(2\delta)-1}$ (line 2 in Table~\ref{tab:decays}), however, as seen in the left plot, the regularization parameter saturates as $n$ grows at the machine precision, leading, because of numerical errors, to worse prediction performance. The problem is so ill-conditioned that the matrix inversion cannot be algorithmically robust enough.

\end{list}

\paragraph{Performance of low-rank approximations.}
In this series of experiments, we compute the rank~$p$ which is necessary to achieve a predictive performance at most $1\%$ worse than with $p=n$, and compute\footnote{Note that in practice, computing the degrees of freedom exactly requires to know the full matrix. However, it could also be  approximated efficiently,
% using a low-rank approximation based on column sampling, 
following for example~\citet{fastdri}.}
 the ratio with the marginal degrees of freedom $d = n   \big\|
\diag \big( K ( K + n \lambda \idm)^{-1}   \big) \big\|_\infty$ and the traditional degrees of freedom $d_{ave} = \tr K^2 (K+n\lambda \idm)^{-2}$. In the right plot of \myfig{rkhs}, we consider data  randomly distributed in $[0,1]$ with the same kernels and functions than above, while in \myfig{m}, we considered three of the \emph{pumadyn} datasets from the UCI machine learning repository (there we compute the classical generalization performance on unseen data points and estimate $\lambda$ by cross-validation). For further experimental evaluations, see~\citet{cortes2010impact,talwalkarmatrix}.

On all datasets, the ratios stay relatively close to one, illustrating the results from Theorem~\ref{theo:perf}. Moreover, the two different versions $d$ and $d_{ave}$ of degrees of freedom are within a factor of 2. Moreover, using pivoting to select the columns does not change significantly the results, but may sometimes reduce the number of required columns by a constant factor. Note that the sudden increase (of magnitude less than 2 in the middle and right plots of \myfig{m}) is due to the chosen criterion (ratio of the sufficient rank to obtain 1\% worse predictive performance), which may be unstable.

\begin{figure}[ht]

 \vspace*{-.4cm}

\centering
\hspace*{-.75cm}
\includegraphics[scale=.5]{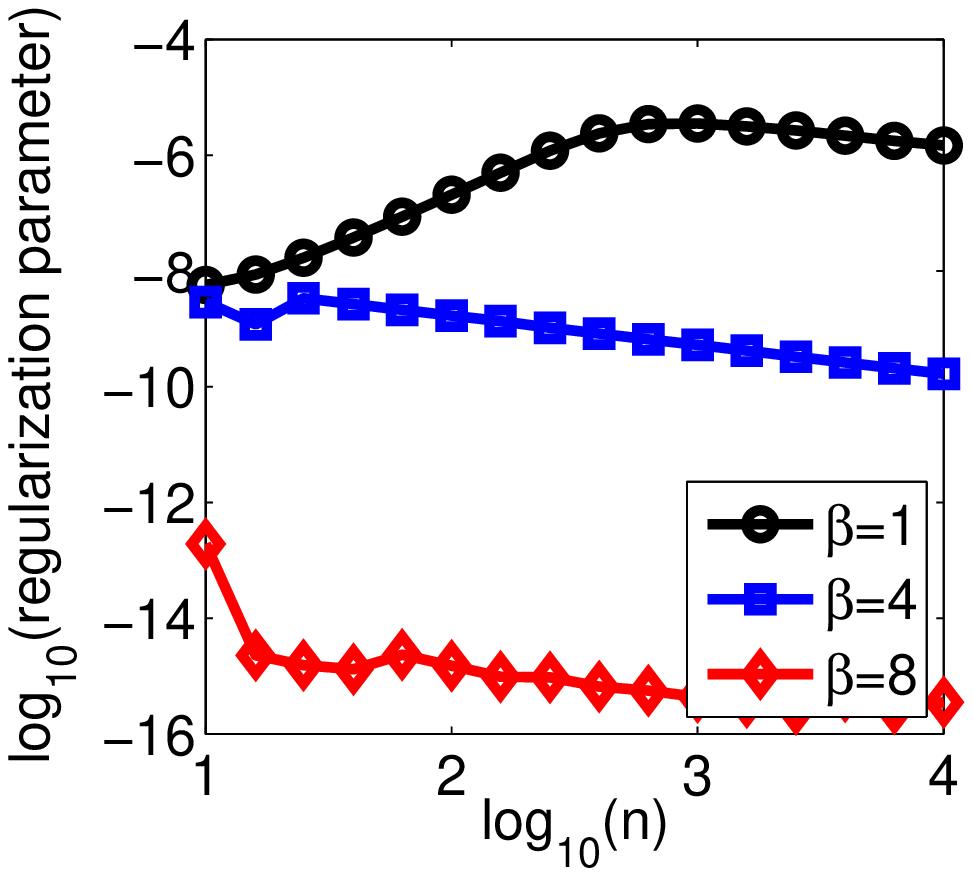} \hspace*{.25cm}
\includegraphics[scale=.5]{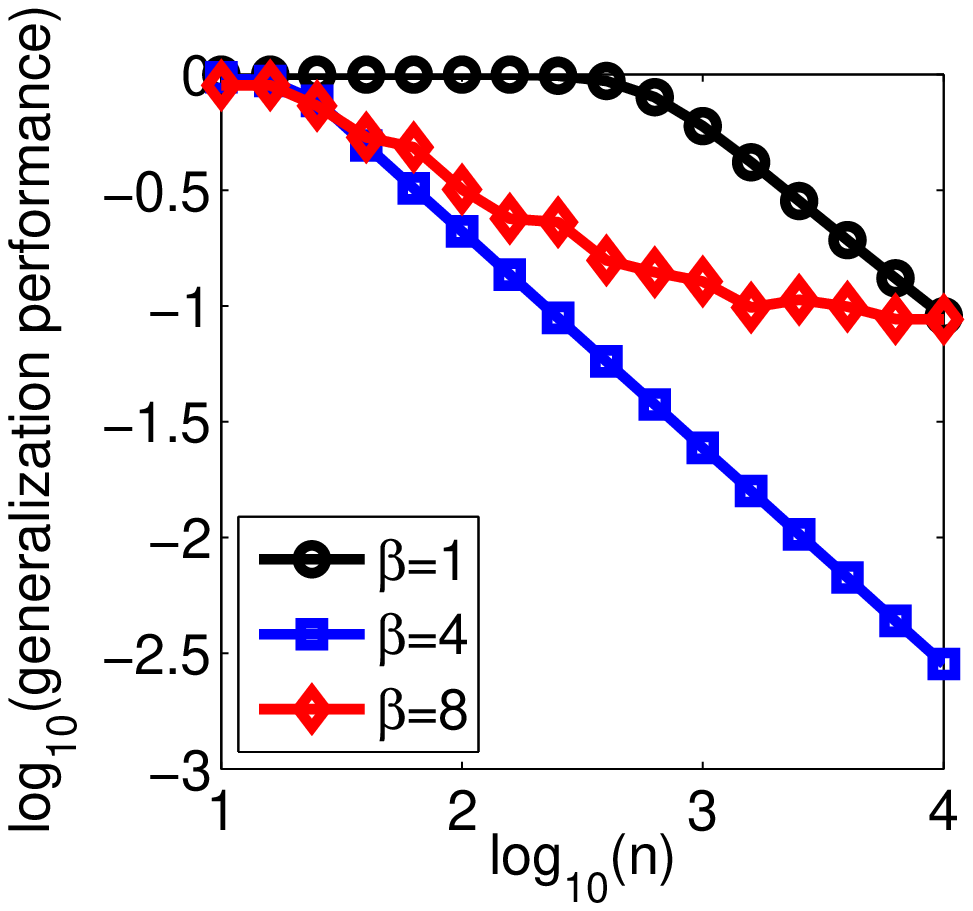} \hspace*{.25cm}
\includegraphics[scale=.5]{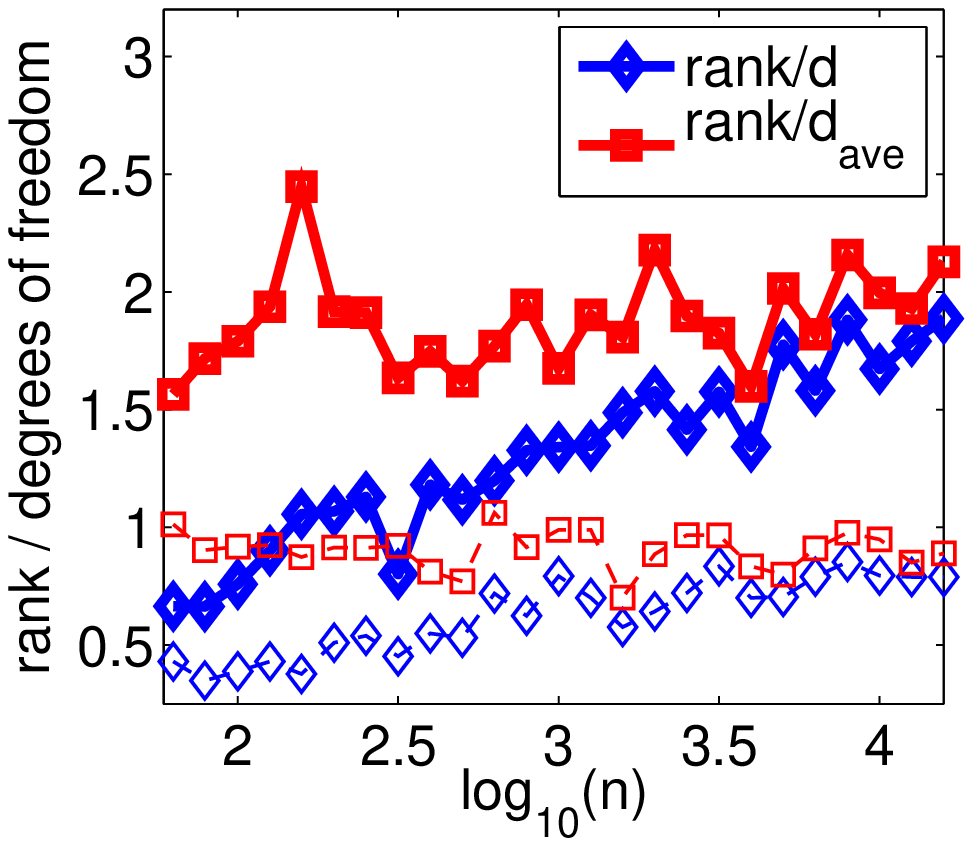}
\hspace*{-.75cm}

 \vspace*{-.5cm}

\caption{Left and middle: Effect of size of RKHS in predictive performance. Right: Ratio of the sufficient rank to obtain $1\%$ worse predictive performance, over the degrees of freedom (plain: random column sampling, dashed: incomplete Cholesky decomposition with column pivoting).}
\label{fig:rkhs}
\end{figure}

\begin{figure}[ht]

 \vspace*{-.4cm}

\centering
\hspace*{-.5cm}
\includegraphics[scale=.5]{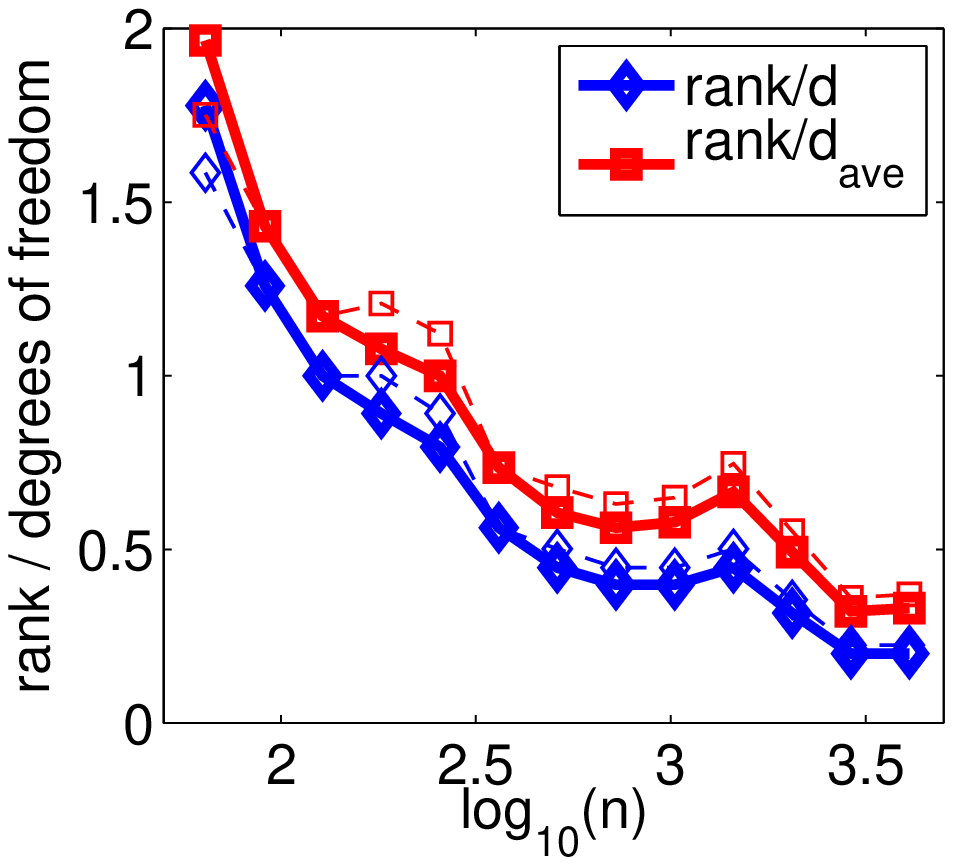} \hspace*{.25cm}
\includegraphics[scale=.5]{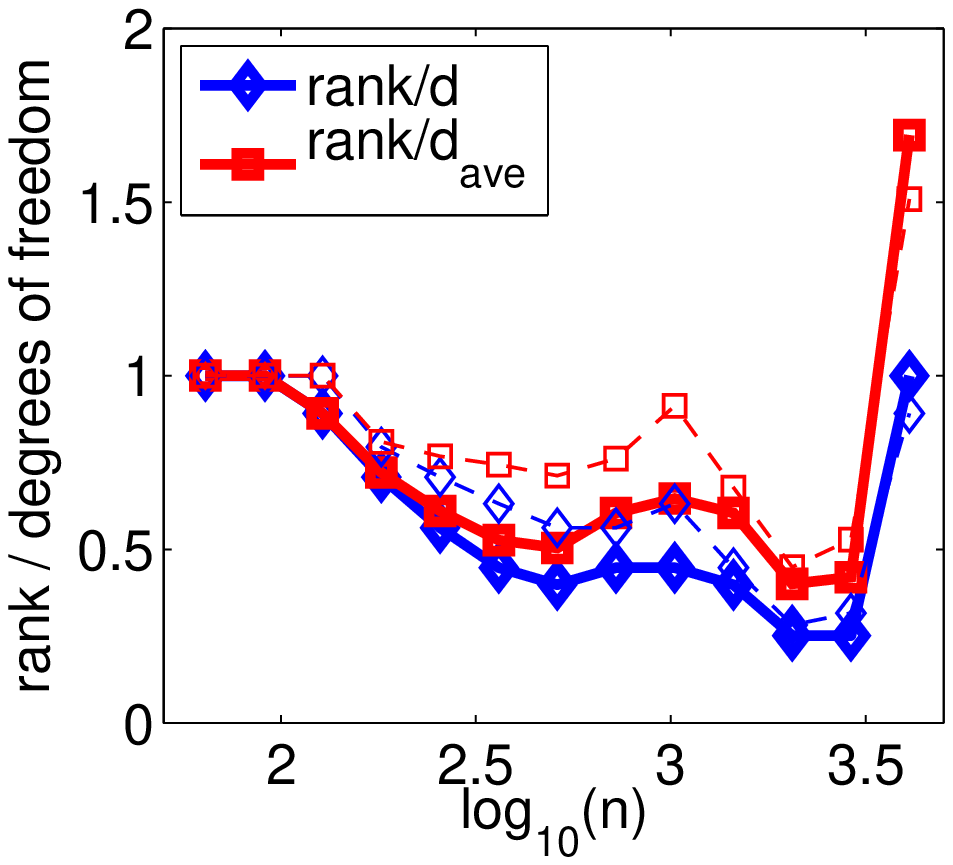} \hspace*{.25cm}
\includegraphics[scale=.5]{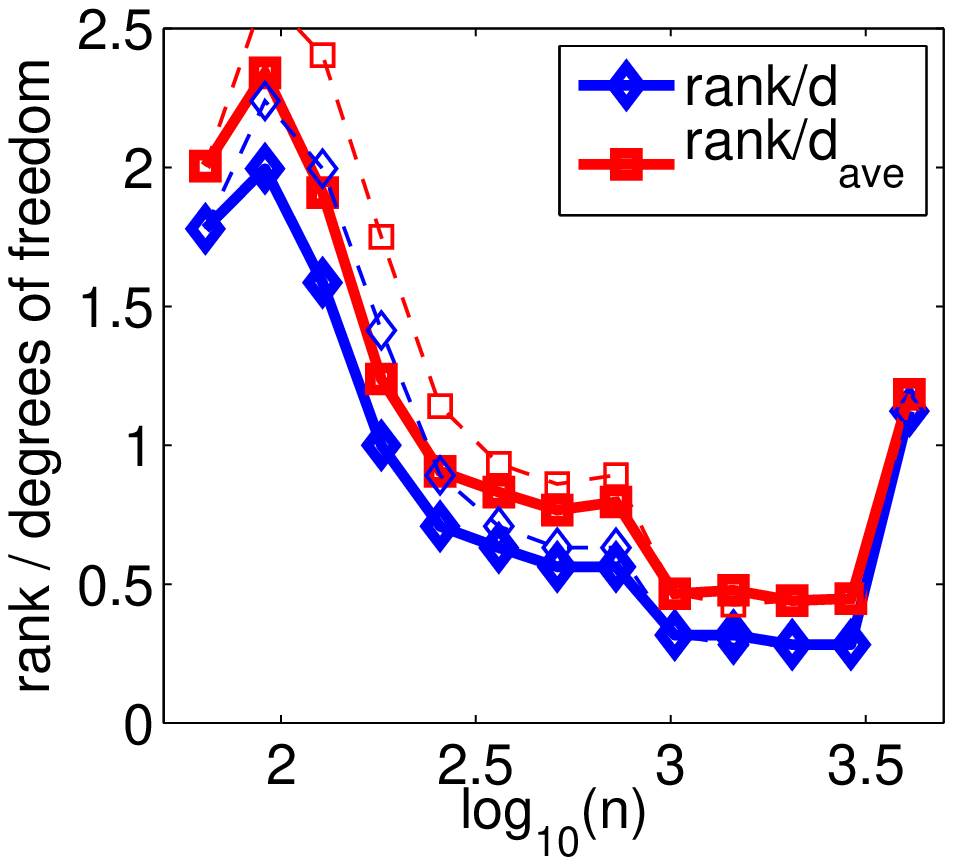}
\hspace*{-.5cm}

 \vspace*{-.5cm}

\caption{Ratio of the sufficient rank to obtain $1\%$ worse predictive performance, over the degrees of freedom (plain: random column sampling, dashed: incomplete Cholesky decomposition with column pivoting). From left to right: \emph{pumadyn} datasets \emph{32fh}, \emph{32nh}, \emph{32nm}.}
\label{fig:m}

\end{figure}

\section{Conclusion}

In this paper, we have provided an analysis of column sampling for   kernel least-squares regression that shows that the rank may be chosen proportional to properly defined degrees of freedom of the problem, showing that the statistical quantity characterizing prediction performance also plays a computational role. The current analysis could be extended in various ways: First, other column sampling schemes beyond uniform, such as presented by~\citet{boutsidis2009improved,ameet}, could be considered with potentially better behavior. Moreover, while we have focused on a fixed design, it is of clear interest to extend our results to random design settings using tools from~\citet{hsu2011analysis}.
The analysis may also be extended to other losses than the square loss, such as the logistic loss, using self-concordant analysis~\citep{bach2010self} or the hinge loss, using eigenvalue-based criteria~\citep{blanchard2008statistical} or tighter approaches to sample complexity analysis~\citep{sabato2010tight}. Finally, in this paper, we have considered a batch setting and extending our results to online settings, in the line of \citet{tarres2011online}, is of significant practical and theoretical interest.

\acks{
This work was supported by the European Research
Council (SIERRA Project). The author would like to thank the reviewers  for suggestions that have helped improve the clarity of the paper.}

\newpage

\appendix

\section{Duality for kernel supervised learning}
\label{app:dual}
In this section, we review classical duality results for kernel-based supervised learning, which extends the dual problem of the support vector machine to all losses. For more details and examples, see~\citet{rifkin2007value}.
We consider the following problem, where $\mathcal{F}$ is an RKHS with feature map
$\phi: \mathcal{X} \to \mathcal{F}$:
$$
\min_{f \in \mathcal{F}} \frac{1}{n} \sum_{i=1}^n \ell(y_i,f(x_i) ) + \frac{\lambda}{2} \| f\|^2,
$$
which may be rewritten with the feature map $\phi$ as:
$$
\min_{f \in \mathcal{F}, u \in \rb^n, } \frac{1}{n} \sum_{i=1}^n \ell(y_i, u_i ) + \frac{\lambda}{2} \| f\|^2 \mbox{ such that } u_i = \langle f, \phi(x_i) \rangle.
$$
We may then introduce dual parameters (Lagrange multipliers) $\alpha \in \rb^n$ and the Lagrangian
$$
\mathcal{L}(f,u,\alpha) = 
\frac{1}{n} \sum_{i=1}^n \ell(y_i, u_i ) + \frac{\lambda}{2} \| f\|^2 + \lambda \sum_{i=1}^n \alpha_i ( u_i -  \langle f, \phi(x_i) \rangle ).
$$
Minimizing with respect to $(f,u)$, we get $f = \sum_{i=1}^n \alpha_i \phi(x_i)$ and the dual problem:
$$
\max_{\alpha \in \rb^n} -g(-\lambda \alpha) - \frac{\lambda}{2} \alpha^\top K \alpha,
$$
where, for $z \in \rb^n$,  $g(z) = \max_{ u \in \rb^n} 
 - \frac{1}{n} \sum_{i=1}^n \ell(y_i,u_i )  + u_i z_i
$ is the Fenchel-conjugate of the empirical risk.

\section{Proof of Theorem~\ref{theo:perf}}
\label{app:proof}

We first prove a lemma that provides a Bernstein-type inequality for subsampled covariance matrices. The proof follows \citet{tropp2010improved,tropp2011user} and \citet{gittens2011spectral}.

\subsection{Concentration of subsampled covariance matrices}
Given the matrix $\Psi \in \rb^{ n\times r }$ and $I \subset \{1,\dots,p\}$, we denote
by $\Psi_I$ the submatrix of $\Psi$ composed of the rows of $\Psi$ indexed by $I$.

\begin{lemma}[Concentration of subsampled covariance]
\label{lemma:concentration}
Let $\Psi \in \rb^{n \times r}$, with all rows of $\ell_2$-norm less than $R$. Let $I$ be a random subset of $\{1,\dots,n\}$ with $p$ elements (i.e., $p$ elements chosen \emph{without} replacement uniformly at random).
Then, for all $t>0$,
$$\PP_I \Big( \lambda_{\max} \Big[  \frac{1}{n} \Psi^\top \Psi -  \frac{1}{p}\Psi_I^\top \Psi_I \Big] > t \Big) \leqslant  r \exp \bigg(
\frac{ - p t^2/2 }{  \lambda_{\max} ( \frac{1}{n} \Psi^\top \Psi) ( R^2  + t / 3 ) }
\bigg).$$
\end{lemma}
\begin{proof}
Let $\psi_1,\dots,\psi_n \in \rb^r$ be the $n$ rows of $\Psi$.
We consider the matrix $\Delta \in \rb^{ r \times r }$ defined as: 
$$
\Delta = 
 \frac{1}{n} \Psi^\top \Psi -  \frac{1}{p}\Psi_I^\top \Psi_I
 = \frac{1}{n} \sum_{i=1}^n \psi_i \psi_i^\top
 -\frac{1}{p} \sum_{i \in I} \psi_i \psi_i^\top.
$$
By construction, we have $\E \Delta = 0$, and, as shown by~\citet{tropp2010improved,tropp2011user} and \citet{gittens2011spectral}, we have
$$
\E \tr \exp ( s \Delta ) \leqslant \E \tr \exp ( s \Xi ) ,
$$
where $\Xi$ is obtained by sampling  independently $p$ rows \emph{with} replacement, i.e., is equal to 
$$ \Xi = \frac{1}{n} \sum_{i=1}^n \psi_i \psi_i^\top
 -\frac{1}{p} \sum_{j=1}^p \sum_{i=1}^n z_i^j  \psi_i \psi_i^\top,
$$
where $z^j \in \rb^n$ is a random element of the canonical basis of $\rb^n$ such that 
$\PP(z^j_i=1) = \frac{1}{n}$ for all $i \in \{1,\dots,n\}$ and $j\in\{1,\dots,p\}$. This result extends to the matrix case the classical result of~\citet{hoeffding}.

We thus have:
$$
\Xi =  \frac{1}{p} \sum_{j=1}^p \bigg(
\frac{1}{n} \sum_{i=1}^n \psi_i \psi_i^\top -
\sum_{i=1}^n z_i^j  \psi_i \psi_i^\top
\bigg) =   \sum_{j=1}^p M_j,
$$
with $M_j = \frac{1}{p} \big(
\sum_{i=1}^n z_i^j  ( \frac{1}{n} \Psi^\top \Psi - \psi_i \psi_i^\top)
\big)$. We have $\E M_j = 0$, 
$\lambda_{\max}(M_j) \leqslant \lambda_{\max}(\frac{1}{n} \Psi^\top \Psi) / p$, and
\BEAS
\lambda_{\max} \bigg(
\sum_{j=1}^p \E M_j^2
\bigg)
& = & 
\frac{1}{p^2}
\lambda_{\max} \bigg(
\sum_{j=1}^p \sum_{i=1}^n \sum_{k=1}^n \E z_i^j  z_k^j  
\big( \frac{1}{n} \Psi^\top \Psi - \psi_i \psi_i^\top\big)\big( \frac{1}{n} \Psi^\top \Psi - \psi_k \psi_k^\top \big)
\bigg) \\
& = & 
\frac{1}{p}
\lambda_{\max} \bigg(
\frac{1}{n}\sum_{i=1}^n   
\big( \frac{1}{n} \Psi^\top \Psi - \psi_i \psi_i^\top\big)^2
\bigg) \mbox{ because } \E z_i^j z_k^j = \frac{1}{n} \delta_{i=k},\\
& = & 
\frac{1}{p}
\lambda_{\max} \big(
\frac{1}{n}\sum_{i=1}^n   \psi_i \psi_i^\top \psi_i \psi_i^\top
- 
 \big( \frac{1}{n} \Psi^\top \Psi \big)^2
\big) \leqslant \frac{1}{p}
\lambda_{\max} \big(
\frac{1}{n}\sum_{i=1}^n   \psi_i \psi_i^\top \psi_i \psi_i^\top
\big) \\
& \leqslant & 
\frac{R^2}{p}
\lambda_{\max} \big(
 \frac{1}{n} \Psi^\top \Psi 
\big) \mbox{ because } 
 \psi_i \psi_i^\top \psi_i \psi_i^\top
 \preccurlyeq   R^2 \psi_i   \psi_i^\top.
\EEAS
We can then apply the matrix Bernstein inequality of~\citet[Theorem 6.1]{tropp2011user} to obtain the probability bound:
$$
r \exp \bigg( - \frac{ t^2/2}{ \displaystyle\frac{R^2}{p} \lambda_{\max} \big(
 \frac{1}{n} \Psi^\top \Psi 
\big) + \frac{1}{p} \lambda_{\max} \big(
 \frac{1}{n} \Psi^\top \Psi 
\big) \frac{t}{3} }\bigg),
$$
which leads to the desired result.

\end{proof}

\subsection{Proof of Theorem~\ref{theo:perf}}
 
 \paragraph{Proof principle.}
 
 Let $\Phi \in \rb^{n \times n}$ be such that  $K = \Phi \Phi^\top$. Note that if $K$ had  rank $r$, we could instead choose  $\Phi \in \rb^{ n \times r}$.
 
 We consider the regularized low-rank approximation $L_\gamma = \Phi N_\gamma \Phi^\top$, with 
 \BEQ
 \label{eq:CC}
 N_\gamma = \Phi_I^\top ( \Phi_I \Phi_I^\top + p \gamma \idm)^{-1} \Phi_I = \Phi_I^\top \Phi_I  (  \Phi_I^\top \Phi_I + p \gamma \idm)^{-1} = \idm - \gamma  (  \Phi_I^\top \Phi_I /p+  \gamma \idm)^{-1}
 \EEQ
 (obtained using the matrix inversion lemma).
 We have $L = L_0$ but we will consider $L_\gamma$ for $\gamma>0$ to obtain a bound for $\gamma=0$, using a monotonicity argument.

Following the same reasoning than in \mysec{df}, the in-sample prediction error $\frac{1}{n} \E_\varepsilon  \| \hat{z}_{L_\gamma} - z\|^2$ is equal to
\BEAS
 \frac{1}{n}\E_\varepsilon  \| \hat{z}_{L_\gamma} - z\|^2 &  = &  n  \lambda^2  \|    (\Phi N_\gamma \Phi^\top + n \lambda \idm)^{-1} z    \|^2   +  \frac{1}{n}\tr  C \big[ \Phi N_\gamma \Phi^\top   (\Phi N_\gamma \Phi^\top  + n \lambda \idm)^{-1} \big]^2 \\
 & = & 
 {\rm bias}(L_\gamma) + {\rm variance}(L_\gamma) .
 \EEAS
 The function $\gamma \mapsto N_\gamma$ is matrix-non-increasing (i.e., if $\gamma \geqslant \gamma'$, then $N_\gamma \preccurlyeq N_{\gamma'}$).
Therefore, we have  $ 0 \preccurlyeq N_\gamma \preccurlyeq N_0 \preccurlyeq \idm$.
Since  the variance term ${\rm variance}(L_\gamma) = \frac{1}{n}\tr  C \big[ \Phi N_\gamma \Phi^\top   (\Phi N_\gamma \Phi^\top  + n \lambda \idm)^{-1} \big]^2$ is non-decreasing in $N_\gamma$, this implies that the variance term with $N_\gamma$ is smaller than the one with $N_0$ and then less then the one with $N_\gamma$ replaced by $\idm$ (which corresponds to the variance term without any approximation). For the bias term we have:
 \BEQ
 \label{eq:bb}
 {\rm bias}(L_\gamma)  =  n \lambda^2  \|    (\Phi N_\gamma \Phi^\top + n \lambda \idm)^{-1} z  \|^2  
 =   n \lambda^2 z^\top     (\Phi N_\gamma \Phi^\top  + n \lambda \idm)^{-2} z   ,
\EEQ
which is a non-decreasing function of $\gamma$. Therefore, if we prove an upper-bound on the bias term for any $\gamma > 0$, we have a bound for $\gamma=0$. This requires \emph{lower-bounding} $N_\gamma$.

\paragraph{Lower-bounding $N_\gamma$.}
Let $\Psi = \Phi ( \frac{1}{n} \Phi^\top \Phi + \gamma \idm)^{-1/2} \in \rb^{n \times n}$. We may rewrite $N_\gamma$ defined in \eq{CC} as
\BEAS
N_\gamma & = & 
\idm - \gamma  (  \frac{1}{p} \Phi_I^\top \Phi_I +  \gamma \idm)^{-1} \\
& = & \idm - \gamma \big(  \frac{1}{n} \Phi^\top \Phi + \gamma \idm  - \frac{1}{n} \Phi^\top \Phi  +  \frac{1}{p} \Phi_I^\top \Phi_I \big)^{-1}
\\
& = & 
\idm  - \gamma
 ( \frac{1}{n} \Phi^\top \Phi + \gamma \idm)^{-1/2}  
 \bigg[
 \idm - \frac{1}{n} \Psi^\top \Psi +  \frac{1}{p}\Psi_I^\top \Psi_I
 \bigg]^{-1}
  ( \frac{1}{n} \Phi^\top \Phi + \gamma \idm)^{-1/2}  .
\EEAS
Thus, in order to obtain a lower-bound on  $N_\gamma$, it suffices to have an upper-bound of the form
\BEQ
\label{eq:DD}
\lambda_{\max} \big(
\frac{1}{n} \Psi^\top \Psi -  \frac{1}{p}\Psi_I^\top \Psi_I
\big) \leqslant t,
\EEQ
which would imply
\BEAS
\idm - N_\gamma & \preccurlyeq & \frac{\gamma}{1-t} \big( \frac{1}{n} \Phi^\top \Phi + \gamma \idm \big)^{-1}, \\
K - L_\gamma & = & \Phi( \idm - N_\gamma) \Phi^\top \preccurlyeq  \frac{\gamma}{1-t} \Phi ( \frac{1}{n} \Phi^\top \Phi + \gamma \idm)^{-1} \Phi^\top
=  \frac{n \gamma}{1-t} ( K + n\gamma \idm)^{-1} K   
\preccurlyeq  \frac{n \gamma}{1-t} \idm.
\EEAS
Assume $ \frac{\gamma / \lambda }{1-t} \leqslant 1$.
We then have, using the previous inequality:
$$
(L_\gamma  + n \lambda \idm)^{-1} 
\preccurlyeq \big( K -    \frac{n \gamma}{1-t} \idm+ n \lambda \idm \big)^{-1} 
= \big( K + n\lambda \big[ 1 - \frac{\gamma/\lambda}{1-t} \big] I\big)^{-1}
\preccurlyeq ( 1- \frac{  \gamma/\lambda}{1-t})^{-1} (K  + n \lambda \idm)^{-1}.
$$
Thus, the bias term in \eq{bb} is   less than the original bias term times
$
( 1 - \frac{ \gamma / \lambda }{1-t})^{-2}
$. If the bound defined in \eq{DD}  is not met, then we can upper-bound the bias term by
$\frac{1}{n} z^\top z$, which is itself upper-bounded by the unapproximated bias term times
$( 1 + \frac{R^2}{\lambda} ) $---indeed, we have
$n  \lambda^2 z^\top (K + n \lambda \idm)^{-2} z
 \geqslant    n  \lambda^2 z^\top z ( n \lambda + n R^2 )^{-2} =  \frac{1}{n}z^\top z ( 1 +  R^2/ \lambda )^{-2}$. Thus if we define $\pi_t = \PP_I \Big( \lambda_{\max} \Big[  \frac{1}{n} \Psi^\top \Psi -  \frac{1}{p}\Psi_I^\top \Psi_I \Big] > t \Big) $, then, 
 $E_I \big[ {\rm bias}(L_\gamma) \big]$ is  upper-bounded by
\BEQ
\label{eq:B}
B = \pi_t   ( 1 + R^2 /\lambda)  + ( 1 - \pi_t)  \big( 1 - \frac{ \gamma / \lambda }{1-t} \big)^{-2}  ,
\EEQ
times ${\rm bias}(K)$.

\paragraph{Probabilistic control.}

We need to upper-bound the largest eigenvalue of
       $
       \frac{1}{n} \Psi^\top \Psi -  \frac{1}{p}\Psi_I^\top \Psi_I
       $,
       where $I$ is a random subset of $\{1,\dots,n\}$ of cardinality $p$. This is the difference between an empirical second-order moment and the empirical moment of a subset of $p$ random elements. In order to apply Lemma~\ref{lemma:concentration}, we  need to upper-bound the squared $\ell_2$-norm as (assuming $\gamma \leqslant \lambda$):
       \BEAS
       \max_{i \in \{1,\dots,n\}}
       \big(
       \Psi \Psi^\top
       \big)_{ii}
       & = &   \max_{i \in \{1,\dots,n\}}
\big(\Phi ( \frac{1}{n} \Phi^\top \Phi + \gamma \idm)^{-1} \Phi^\top \big)_{ii}\\
& = & \lambda \gamma^{-1} \max_{i \in \{1,\dots,n\}}
\big( \Phi ( \frac{1}{n} (\lambda \gamma^{-1})\Phi^\top \Phi + \lambda \idm)^{-1} \Phi ^\top  \big)_{ii} \\
& \leqslant &\lambda \gamma^{-1} \max_{i \in \{1,\dots,n\}}
\big(  \Phi ( \frac{1}{n} \Phi^\top \Phi + \lambda \idm)^{-1} \Phi^\top  \big)_{ii}
\mbox{ because } \gamma \leqslant \lambda,
\\
& = & n \lambda \gamma^{-1} \big\|
\diag \big( K ( K + n \lambda \idm)^{-1}   \big) \big\|_\infty = \lambda\gamma^{-1} d.
\EEAS
       Thus for $\gamma \leqslant \lambda$,  all rows of $\Psi$ have a squared $\ell_2$-norm upper-bounded by $\lambda \gamma^{-1} d$, and $ \frac{1}{n} \Psi^\top \Psi  \preccurlyeq \idm$, we can apply Lemma~\ref{lemma:concentration}, to obtain that:     
\BEQ
\label{eq:proba} \pi_t = \PP_I \Big( \lambda_{\max} \Big[  \frac{1}{n} \Psi^\top \Psi -  \frac{1}{p}\Psi_I^\top \Psi_I \Big] > t \Big) \leqslant  n \exp \bigg(
\frac{ - p t^2/2 }{ \lambda \gamma^{-1} d   + t / 3 }
\bigg).\EEQ

Using the bound from \eq{B}, we get, given $\delta \in (0,1)$, $ t= 1/2$, and $\gamma   = \frac{\lambda \delta }{4}$ (which satisfy $\gamma \leqslant \lambda$
and $\frac{ \gamma/ \lambda}{1-t} = \delta/2 \leqslant 1$),
\BEAS
B & = & \big( 1 + \frac{R^2}{\lambda} \big) \pi_t + ( 1 - \pi_t ) 
\bigg[ \big( 1 - \frac{ \gamma / \lambda }{1-t} \big)^{-2}   - 1 \bigg] \\
& \leqslant &  1 +\frac{n R^2}{\lambda} \exp \bigg(
\frac{ -     p / 8   }{ 4 d / \delta  + 1/6   }
\bigg) + \bigg[
( 1 - \delta/2)^{-2} - 1
\bigg] \\
& \leqslant & 1 + \frac{n R^2}{\lambda} \exp \bigg(
\frac{ -     p     }{ 32 d / \delta  + 2   }
\bigg) + \bigg[ \frac{ \delta  - \delta^2/4}{( 1 - \delta/2)^2}
\bigg] \\
& = &  1 +\frac{n R^2}{\lambda} \exp \bigg(
\frac{ -     p     }{ 32 d / \delta  + 2   }
\bigg) + \delta \bigg[ \frac{ 1  - \delta/4}{( 1 - \delta/2)^2}
\bigg] \\
& \leqslant &  1 +\frac{n R^2}{\lambda} \exp \bigg(
\frac{ -     p     }{ 32 d / \delta  + 2   }
\bigg) + \delta \bigg[ \frac{ 3/4}{1/4}
\bigg] =  1 + \frac{n R^2}{\lambda} \exp \bigg(
\frac{ -     p     }{ 32 d / \delta  + 2   }
\bigg) + 3 \delta  .
\EEAS
Thus, if $p \geqslant \big( \frac{32 d }{\delta} + 2 \big) \log \frac{n R^2}{\delta \lambda}$, we obtain that $B \leqslant 1 + 4 \delta$.

Thus,
\BEAS
\E_I \big[ {\rm bias}(L) + {\rm variance}(L)  \big]
& \leqslant & \E_I \big[ {\rm bias}(L_\gamma)  + {\rm variance}(K)  \big] \mbox{ by monotonicity}
\\
& = & \E_I \big[ {\rm bias}(L_\gamma)]   + {\rm variance}(K)   \\
& \leqslant & (1+ 4\delta) {\rm bias}(K)   + {\rm variance}(K)     \\
& \leqslant &  (1+ 4\delta)  \big[{\rm bias}(K)   + {\rm variance}(K)  \big],\EEAS
 which is the desired result. Note that
 \BIT
 \item[--] We could improve the bound by expliciting the reduction of the variance term.
 \item[--] In some situations, the prediction performance for the approximated version may in fact be smaller than the non-approximated version.
 \item[--] The proof technique relies on a high-probability bound with respect to the sampling of columns. More precisely, from \eq{B} and \eq{proba}, with $t = 1/2$ and $\gamma = \frac{\lambda\delta}{4}$, we get:
 \BEQ
 \label{eq:high}
  \PP \big(  \frac{1}{n}\E_\varepsilon  \| \hat{z}_L - z\|^2
 \geqslant   \big( 1 -  \textstyle \frac{\delta}{2} \displaystyle \big)^{-2}  \frac{1}{n}\E_\varepsilon  \| \hat{z}_K - z\|^2
 \big) \leqslant  n \exp \bigg(
\frac{ -     p     }{ 32 d / \delta  + 2   }
\bigg)  .
 \EEQ
 \EIT

\section{Asymptotics of bias and variance terms}

\label{app:lambda}

In this appendix, we consider various decays of eigenvalues $n\mu_i$ of $K$ and components $\sqrt{ n \nu_i}$ (in magnitude) of the signal $z$ to estimate. We follow the reasoning of~\citet{Har_Bac_Mou:2008}, i.e., replacing sums by integrals. Given our assumptions, we have:
\BEAS
{\rm bias} & = &  n^2 \lambda^2 \sum_{i=1}^n \frac{ n \nu_i }{ (n \mu_i + n \lambda)^2} =  n \lambda^2 \sum_{i=1}^n \frac{   \nu_i }{ (\mu_i +  \lambda)^2} ,\\
\frac{n}{\sigma^2} {\rm variance} & = & \sum_{i=1}^n \frac{ n^2 \mu_i^2}{ (n \mu_i + n \lambda)^2} 
= \sum_{i=1}^n \frac{  \mu_i^2}{ (  \mu_i +   \lambda)^2} .
\EEAS

For all cases we need to consider, for simplicity, we only provide an upper-bound for $\mu_i$ exactly equal to its asymptotic equivalent. Considering lower-bounds and a constant times the asymptotic equivalent may be done in a similar way.

\subsection{Variance terms}
 We consider the only two possible cases (the variance term only depends on $(\mu_i)$). Moreover we show that the two traditional definitions of the degrees of freedom, $\tr K ( K + n \lambda \idm)^{-1}$ and $\tr K^2 ( K + n \lambda \idm)^{-2}$, have the same asymptotically equivalents.
 
\paragraph{Polynomial decay ($\mu_i = i^{-2 \beta}$, $\beta>1/2$). }

The renormalized variance term is less than
\BEAS
\sum_{i=1}^n \frac{  1}{ (  1 + i^{2 \beta}   \lambda)^2} 
& \leqslant & 
\int_{0}^n \frac{  1}{ (  1 + t^{2 \beta}   \lambda)^2} dt \\
& = & 
\int_{0}^{\lambda n^{2\beta}} \frac{  1}{ (  1 + u)^2 }  {\lambda}^{-1/2\beta} u^{ 1/2\beta - 1} \frac{1}{2\beta} du  
\mbox{ with the change of variable }  u = \lambda t^{2\beta}, \\
& \leqslant & \int_{0}^{\infty} \frac{  1}{ (  1 + u)^2 }  {\lambda}^{-1/2\beta} u^{ 1/2\beta - 1} \frac{1}{2\beta} du   \\
& = & O ( {\lambda}^{-1/2\beta}) \mbox{ since the integral is finite.}
\EEAS

With the same reasoning, we have
$\tr K ( K + n \lambda \idm)^{-1} \leqslant 
\int_{0}^{\lambda n^{2\beta}} \frac{  1}{ (  1 + u) }  {\lambda}^{-1/2\beta} u^{ 1/2\beta - 1} \frac{1}{2\beta} du  = O ( {\lambda}^{-1/2\beta}) .
$

\paragraph{Exponential decay ($\mu_i = e^{ - \rho i }$). }

The renormalized variance term is less than
\BEAS
\sum_{i=1}^n \frac{  1}{ (  1 +e^{\rho i }\lambda)^2} 
& \leqslant & 
\int_{0}^n \frac{  1}{ (  1 +e^{\rho t }\lambda)^2} dt
= 
\int_{0}^n \frac{  e^{-2 \rho t }  }{ (  e^{-\rho t } +\lambda)^2} dt \\
& = & \frac{1}{ \rho}
\int_{e^{- \rho n}}^1 \frac{  u }{ (  u  +\lambda)^2} du 
\mbox{ with the change of variable } u =e^{-\rho t } \\
& \leqslant & \frac{1}{ \rho}
\int_{0}^1  \bigg( \frac{  1 }{  u  +\lambda} -  \frac{  \lambda }{ (  u  +\lambda)^2} \bigg) du  \leqslant  \frac{1}{ \rho}
\int_{0}^1    \frac{  1 }{  u  +\lambda}  du  \\
& = & \frac{1}{\rho} \big[  \log ( 1 + \lambda)   - \log \lambda  \big] = O( \log \frac{1}{\lambda} ).
\EEAS
We the same technique, we get bounds on $\tr K ( K + n \lambda \idm)^{-1}$ in the same way we just did for $\tr K^2 ( K + n \lambda \idm)^{-2}$.

\subsection{Bias terms}

The bias terms depend on both $(\mu_i)$ and $(\nu_i)$ and we consider all combinations.

\paragraph{Polynomial decays ($\mu_i = i^{-2 \beta}$,  $\nu_i = i^{-2 \delta}$,
$ \beta, \delta > 1/2$). }
The bias term is less than
\BEQ
\label{eq:AA}
n \lambda^2 \sum_{i=1}^n \frac{   i^{4 \beta -2 \delta} }{ (1 + i^{2 \beta}  \lambda)^2} 
\leqslant
 2 n \lambda^2
\int_1^n
\frac{   t^{4 \beta -2 \delta} }{ (1 + t^{2 \beta}  \lambda)^2} dt.
\EEQ

If $2 \delta - 4 \beta > 1$, then we have  an upper bound of 
$2 n \lambda^2
\int_1^\infty
    t^{4 \beta -2 \delta}  dt = O( n \lambda^2 )$, because the integral is finite.

If $2 \delta - 4 \beta < 1$, then we can further bound \eq{AA} as
\BEAS
2 n \lambda^2
\int_1^n
\frac{   t^{4 \beta -2 \delta} }{ (1 + t^{2 \beta}  \lambda)^2} dt
& = & 
2 n \lambda^2
\int_\lambda ^{n^{2\beta} \lambda}
\frac{   u^{2 - \delta/ \beta + \frac{1}{2\beta} - 1 }
\lambda^{-2 + \delta/ \beta - \frac{1}{2\beta} }}{ (1 + u)^2} \frac{1}{2\beta} du \\
& & 
\mbox{ with the change of variable }  u = \lambda t^{2\beta} \\
&  = & O( \lambda^{ (2\delta - 1)/ 2 \beta  })
\int_0 ^\infty 
\frac{   u^{2 - \delta/ \beta + \frac{1}{2\beta} - 1 }
 }{ (1 + u)^2} \frac{1}{2\beta} du = O( \lambda^{ (2\delta - 1)/ 2 \beta  }),
\EEAS
because the integral is finite (due to the assumptions made on $\beta$ and $\delta$).

\paragraph{Exponential decays ($\mu_i = e^{ - \rho i }$, $\nu_i = e^{ - \kappa i }$,
$\rho,\kappa>0$). }

The bias term is less than
\BEAS
 n \lambda^2 \sum_{i=1}^n \frac{  e^{(2 \rho -\kappa)i }}{ (1 +  e^{\rho i} \lambda)^2} 
 & \leqslant & n \lambda^2  \int_1^n 
 \frac{  e^{(  \rho -\kappa)t }}{ (1 +  e^{\rho t} \lambda)^2} e^{  \rho t} dt\\
& = &  \frac{n \lambda }{\rho}\int_\lambda ^{\lambda e^{n\rho}} 
 \frac{ (u/\lambda)^{1 - \kappa/\rho} }{ (1 +  u  )^2} d u
 \mbox{ with the change of variables } u = \lambda e^{\rho t}.
\EEAS
If $\kappa / \rho > 2$, then we have a bound
$$
\frac{n \lambda^2 }{\rho}\int_1 ^{  \infty} 
 \frac{ u^{1 - \kappa/\rho} }{ (1 +  \lambda u  )^2} du = O ( n \lambda^2) ,
$$
because the integral is finite and uniformly bounded in $\lambda$.

If $\kappa / \rho < 2$, then we have a bound
$$
 \frac{n \lambda^{ \kappa/\rho} }{\rho}\int_\lambda ^{\lambda e^{n\rho}} 
 \frac{ u^{1 - \kappa/\rho} }{ (1 +  u  )^2} d u
 \leqslant 
 \frac{n \lambda^{ \kappa/\rho} }{\rho}\int_0 ^{\infty} 
 \frac{ u^{1 - \kappa/\rho} }{ (1 +  u  )^2} d u = O (n \lambda^{ \kappa/\rho} ) .
$$

\paragraph{Mixed decays.}
For $\mu_i$ with polynomial decays and $\nu_i$ with exponential decays,  we are in a situation where $\nu_i$ is decaying fast enough (faster than $i^{-2\delta}$ for any $\delta>1/2$) so that, given previous results, the bias is $n\lambda^2$. 

The only remaining result to show is $\mu_i = e^{ - \rho i}$ and $\nu_i = i^{-2 \delta}$, 
$\delta>1/2$, which we now consider. The bias term is equal to
\BEAS
n \lambda^2 \sum_{i=1}^n \frac{   \nu_i }{ (\mu_i +  \lambda)^2} 
& = & n \lambda^2 \sum_{i=1}^n \frac{   i^{-2 \delta} }{ (e^{ - \rho i}+  \lambda)^2}  \\
& = & n \lambda^2 \sum_{ i \leqslant \frac{1}{\rho} \log \lambda^{-1} } \frac{   i^{-2 \delta} }{ (e^{ - \rho i}+  \lambda)^2} 
+  n \lambda^2 \sum_{ n \geqslant i > \frac{1}{\rho} \log \lambda^{-1} } \frac{   i^{-2 \delta} }{ (e^{ - \rho i}+  \lambda)^2}  \\
& \leqslant & n \lambda^2 \sum_{ i \leqslant \frac{1}{\rho} \log \lambda^{-1} } \frac{   i^{-2 \delta} }{    e^{ - 2\rho i } } 
+  n \lambda^2 \sum_{ n \geqslant i > \frac{1}{\rho} \log \lambda^{-1} } \frac{   i^{-2 \delta} }{ \lambda^2 }  \\
& \leqslant & n   \sum_{ i \leqslant \frac{1}{\rho} \log \lambda^{-1} }    i^{-2 \delta}  +  n   \sum_{   i > \frac{1}{\rho} \log \lambda^{-1} }  {   i^{-2 \delta} }{   }  \\
& = & O(n) + O( n (  \log \lambda^{-1})^{ 1- 2 \delta} ) = O( n (  \log \lambda^{-1})^{ 1- 2 \delta} ).
%& \leqslant & n \lambda^2  \frac{1}{\lambda^2 \rho}  \int_1^n   \frac{   t^{-2 \delta} \lambda e^{  \rho t} }{ (1+  e^{  \rho t}\lambda)^2}  \lambda \rho e^{  \rho t} dt \\
%& = &  \frac{n}{   \rho}  \int_{\lambda}^{\lambda e^{\rho n}}   \frac{   ( \frac{1}{\rho} \log \frac{u}{ \lambda \rho}) ^{-2 \delta}u }{ (1+ u)^2} du  \\
%&  = &   O ( n (\log \frac{1}{\lambda} )^{  -2 \delta}) \int_{\lambda}^{\lambda e^{\rho n}} \frac{u}{u+1} du = O \bigg( n (\log \frac{1}{\lambda} )^{ 1 -2 \delta}\bigg).
\EEAS

 \subsection{Optimal regularization parameters}

We can now take all six cases, and compute the optimal $\lambda$ and the resulting optimal regularization error.

\BIT
\item[--] $\mu_i = i^{-2\beta} $, $\nu_i =  i^{-2\delta}$ ($2\delta > 4 \beta + 1$): we need to minimize with respect to $\lambda$ the function
$n^{-1}\lambda^{-1/2\beta} +   \lambda^2
$, which leads to $\lambda \approx n^{ - 1/ ( 2 + 1/2\beta)}$ and an optimal value of
$n^{  1/ ( 4 \beta + 1)-1}$.

\item[--] $\mu_i = i^{-2\beta} $, $\nu_i =  i^{-2\delta}$ ($2\delta < 4 \beta + 1$): we need to minimize with respect to $\lambda$ the function
$n^{-1} \lambda^{-1/2\beta} +   \lambda^{ ( 2 \delta -1 )/2\beta}
$, which leads to $\lambda \approx n^{ - \beta / \delta}$ and an optimal value of
$n^{  1/ ( 2\delta) - 1}$.

\item[--] $\mu_i = i^{-2\beta} $, $\nu_i =  e^{ - \kappa i}$: same computation as the first one.

\item[--] $\mu_i = e^{- \rho i}$, $\nu_i =  i^{-2\delta}$: we need to minimize with respect to $\lambda$ the function
$n^{-1}\log \frac{1}{\lambda} +   ( \log \frac{1}{\lambda})^{1 -2 \delta}
$, which leads to $\log \frac{1}{\lambda} \approx n^{ 1 / 2\delta}$ and an optimal value of
$n^{ 1 / 2\delta-1}$.

\item[--] $\mu_i = e^{- \rho i}$, $\nu_i =  e^{ - \kappa i}$ ($\kappa> 2 \rho$): we need to minimize with respect to $\lambda$ the function
$n^{-1} \log \frac{1}{\lambda} +   \lambda^2 $, which leads to $\lambda \approx n^{ -1 / 2 }$ and an optimal value of
$\log n/n$.

\item[--] $\mu_i = e^{- \rho i}$, $\nu_i =  e^{ - \kappa i}$ ($\kappa< 2 \rho$): we need to minimize with respect to $\lambda$ the function
$n^{-1} \log \frac{1}{\lambda} +  \lambda^{ \kappa / \rho} $, which leads to $\lambda \approx n^{ - \rho/\kappa }$ and an optimal value of
$\log n/n$.

\EIT

\section{Kernels on [0,1]}
\label{app:kernels}

In this appendix, we consider kernels on $\mathcal{X} = [0,1]$ that lead to closed-form expressions   (or asymptotic equivalents) for eigenvalues of $K$ and components of $z$. These are used in simulations. 

\paragraph{Kernels.}
For a positive summable sequence $(\mu_i)_{i \geqslant 1}$, we consider 
$k(x,y) = \sum_{i=1}^\infty 2 \mu_i \cos 2 i \pi (x-y)$. It is defined for any $(x,y) \in [0,1]^2$ and is 1-periodic in $x$ and $y$. It is a function $g$ of $
x - y - \lfloor x - y \rfloor$, i.e., $k(x,y) = g( x - y - \lfloor x - y \rfloor )$. Moreover $k(x,x)$ is independent of $x$.

For $\mu_i = \frac{1}{i^{2\beta}}$, we have  $k(x,y) = \frac{1}{(2\beta)!}
B_{2 \beta}(x-y - \lfloor x - y\rfloor) $,
 where
 $B_{2\beta}$ is the $(2\beta)$-th Bernoulli polynomial~\citep{wahba}. For example, we have $B_2(x) =   \textstyle x^2 - x + \frac{1}{6} $ and $
  B_6(x) =  \textstyle x^6 - 3x^5 + \frac{5}{2}x^4  - \frac{1}{2} x^2 + \frac{1}{42}$.
 
For $\mu_i = e^{-\rho i}$, we have, $k(x,y) = 2  \frac{ e^\rho \cos 2 \pi (x-y) - 1  }{ e^{2 \rho} - 2 e^\rho \cos 2 \pi (x-y) + 1    } $. Indeed,  we have
\BEAS
k(x,y) & = &  {\rm Re} \bigg( \sum_{i=1}^\infty 2 e^{-\rho i +  2 i \omega \pi (x-y) } \bigg)
\mbox{ with } \omega^2 = -1, \\
 & = &  2 {\rm Re} \bigg( \sum_{i=1}^\infty   e^{ i ( -\rho +  2   \omega \pi (x-y) ) } \bigg) =  2 {\rm Re} \bigg(  \frac{e^{   -\rho +  2   \omega \pi (x-y)  }}{
 1 - e^{   -\rho +  2   \omega \pi (x-y) } }\bigg) \\
 & = &  2 {\rm Re} \bigg(  \frac{1 }{
  e^{   \rho -  2   \omega \pi (x-y) } - 1  }\bigg) = 2  \frac{ e^\rho \cos 2 \pi (x-y) - 1  }{ e^{2 \rho} - 2 e^\rho \cos 2 \pi (x-y) + 1    } .
\EEAS 

\paragraph{Data and eigenvectors.}
If we consider $n$ data points $x_i = \frac{i-1}{n}$, $i=1,\dots,n$, then the kernel matrix 
$K$ has components $K_{ij} = k( \frac{i-1}{n},\frac{j-1}{n})$. It is a circulant matrix, thus it is diagonalizable in the discrete Fourier basis~\citep{gray}, with eigenvalues equal to the discrete Fourier transform of the first column of the matrix, i.e.,
$(g(0),g(1/n),\dots,g(1-1/n))^\top$.

Thus, the $i$-th eigenvector has $j$-th component $  \frac{1}{\sqrt{n}} e^{ 2\omega i (j-1) \pi / n } $ (with $\omega^2 = -1$) and the $i$-th eigenvalue is
\BEAS
\lambda_i & = & 
\sum_{j=1}^n   e^{ - 2\omega i (j-1) \pi / n } g ( (j-1)/n ) \\
&  = & 
\sum_{j=1}^n   e^{ - 2\omega i (j-1) \pi / n }  
\sum_{s=1}^\infty 2 \mu_s \cos 2 s \pi (j-1)/n  \\
&  = & 
\sum_{j=1}^n   e^{ - 2\omega i (j-1) \pi / n }  
\sum_{s=1}^\infty   \mu_s [ e^{ 2 s \omega \pi (j-1)/n} +e^{ -2 s \omega \pi (j-1)/n}  ]  \\
& = & n \sum_{s = 1}^\infty \mu_s 
\big[ \delta_{ s = i  [n] } + \delta_{ -s = i   [n] } \big] 
\mbox{ because of the orthonormality of the Fourier basis}, \\
& = & 
n \mu_i  + 
n \sum_{ h = 1}^\infty \mu_{i + h n}
+ n  \sum_{ h = 1}^\infty \mu_{ -i + h n}.
\EEAS
 If $n$ is large and $\mu_i$ tends to zero when $i$ tends to $+\infty$, then an  asymptotic equivalent for $\lambda_i$ is $n \mu_i$. 
 
 For data sampled from the uniform distribution in $[0,1]$, then similar equivalents hold~\citep[see, e.g.,][]{Har_Bac_Mou:2008}.

 \paragraph{Functions.} Let $f(x)
= \sum_{i=1}^\infty 2 \nu_i^{1/2} \cos 2 i\pi x$, for $\nu_i $ a non-negative summable sequence. We consider $z_i = f(x_i) = f( (i-1)/n)$. The component of $z$ on the $i$-th eigenvector of $K$ is (following the same reasoning as above):
\BEAS
& & \sum_{j=1}^n  \frac{1}{\sqrt{n}} e^{ - 2\omega i (j-1) \pi / n } f( (j-1)/n) \\
& = &  {\sqrt{n}} \bigg(
\nu_i^{1/2}  + 
\sum_{ h = 1}^\infty \nu_{i + h n}^{1/2}
+ \sum_{ h = 1}^\infty \nu_{ -i + h n}^{1/2}
\bigg),
\EEAS
and the asymptotic equivalent is $(n\nu_i)^{1/2}$.

\paragraph{Link with Sobolev spaces.}
The kernel $k(x,y)$ defined above corresponds for $\mu_i = i^{-2\beta}$ to certain Sobolev spaces~\citep{wahba,gu2002smoothing}. Indeed, for $\beta$ integer, the associated RKHS is the Sobolev space of periodic functions which are $\beta$-times differentiable.

Moreover, when $\nu_i = i^{-2 \delta}$, then for $\delta>\delta_0$, then the corresponding function is $(\delta_0 - 1/2)$-times differentiable, and the minimax rate of estimation is known to  be exactly $O(n^{1/2\delta_0})$~\citep{speckman1985spline,johnstone1994minimax}. Thus, up to logarithmic terms, the best possible rate is $O(n^{1/2\delta})$, and is achieved if $\beta$ is large enough (\mysec{decay}).

\bibliography{columnsampling}

\end{document}